\documentclass[letterpaper, 10 pt, conference]{ieeeconf} 
\IEEEoverridecommandlockouts 

\usepackage{mathrsfs,amsfonts,amsmath,amssymb}

\usepackage{amsthm}
\usepackage{caption}
\usepackage{graphicx,subcaption,tikz}
\usepackage{makecell}
\usepackage{color}
\usepackage{hyperref}
\graphicspath{{./figures/}}
\usepackage{balance}
\usepackage[sort,compress]{cite}
\providecommand{\citet}[1]{\citeauthor{#1}\,[\citeyear{#1}]}
\providecommand{\citep}[1]{\cite{#1}}

\newtheorem{thm}{Theorem}

\newtheorem{Problem}{Problem}

\newcommand{\argmax}[1]{\underset{#1}{\operatorname{argmax}}\;}

\newcommand{\matblock}[1]{\left[\begin{array}#1\end{array}\right]}

\allowdisplaybreaks[1]

\title{\LARGE \bf Information-based Active SLAM via Topological Feature Graphs}
\author{Beipeng Mu$^1$ \quad Matthew Giamou$^1$ \quad Liam Paull$^2$ \quad Ali-akbar Agha-mohammadi$^3$ \\ John Leonard$^2$ \quad Jonathan How$^1$
\thanks{$^{1}$Laboratory for Information and Decision Systems,
	MIT, 77 Mass Ave, Cambridge, MA, USA
	{\tt\small \{mubp, mgiamou, jhow\}@mit.edu} }%
\thanks{$^{2}$Computer Science and Artificial Intelligence Laboratory, MIT, 77 Mass Ave, Cambridge, MA, USA,
	{\tt\small \{lpaull, jleonard\}@mit.edu} }
\thanks{$^{3}$Qualcomm Research, 5775 Morehouse Drive, San Diego, CA, USA,
	{\tt\small aliagha@qualcomm.com }}
}

\begin{document}
\maketitle

\begin{abstract}
Active SLAM is the task of actively planning robot paths while simultaneously building a map and localizing within.
Existing work has focused on planning paths with occupancy grid maps, which do not scale well and suffer from long term drift.
This work proposes a Topological Feature Graph (TFG) representation that scales well and develops an active SLAM algorithm with it.
The TFG uses graphical models, which utilize independences between variables, and enables a unified quantification of exploration and exploitation gains with a single entropy metric. Hence, it facilitates a natural and principled balance between map exploration and refinement. 
A probabilistic roadmap path-planner is used to generate robot paths in real time.
Experimental results demonstrate that the proposed approach achieves better accuracy than a standard grid-map based approach while requiring orders of magnitude less computation and memory resources.  
\end{abstract}

\section{Introduction} \label{sec:introduction}

The exploration of an unknown space is a fundamental capability for a mobile robot, with diverse applications such as disaster relief, planetary exploration, and surveillance. In the absence of a global position reference (e.g., GPS) the robot must simultaneously map the space and localize itself within that map, referred to as SLAM. If a mobile robot is able to successfully recognize parts of the map when it returns to them, referred to as loop closure, then it can significantly reduce its mapping and localization error. The problem of active SLAM focuses on designing robot trajectories to actively explore an environment and minimize the map error.

Previous work has been done on designing trajectories to reduce robot pose uncertainty when the map is known or there exists a global position reference \cite{prentice2009belief, Huang_IROS_2009, Blackmore_TRO_2011, Ali14_IJRR,  vandenBerg_IJRR_2011}. There also exists work that maximizes myopic information gain on the next action with partially known maps \cite{VidalCalleja10_activeslam, Martinez-Cantin2009}. However, in this work, the goal is to build a map of an unknown environment thus the robot needs to plan its path and perform SLAM at the same time (active SLAM) with a focus on \textit{global} map quality.
Active SLAM is non-trivial because the robot must trade-off the benefits of \textit{exploring} new areas and \textit{exploiting} visited areas to close loops\cite{Yamauchi97}.

Previous work has heavily relied on the occupancy grid (OG) map (grid of independent binary random variables denoting occupancy) to compute the information gain on map exploration and check feasibility
For example, the seminal work of Bourgault et al. \cite{Bourgault_IROS_2002} formulates the problem as a trade-off between information gain about the map and entropy reduction over the robot pose:
\begin{equation}
u^* = \max_u w_1I_{SLAM}(x,u) + w_2I_{OG}(x, u)
\label{eq:active_slam}
\end{equation}
where $I_{OG}$ is the information gained over the occupancy grid (OG) map (grid of independent binary random variables denoting occupancy) and $I_{SLAM}$ is the information gained over of the robot poses (dependent Gaussian random variables). Similarly, Stachniss et al.\cite{Stachniss_RSS_2005} use a Rao-Blackwellized particle filter (RBPF) to represent the robot poses and the map, and then consider the informativeness of actions based on the expected resultant information gain. Other information metrics within a similar framework, such as the Cauchy-Schwarz quadratic mutual information \cite{Kumar_RSS_15}, the D-optimality criterion \cite{Carillo_ICRA_2012}, and the Kullback-Leibler divergence \cite{Carlone_IROS_2012} have also been proposed recently. These two information gains are computed separately and maintaining the balance between them often requires careful parameter tuning on the weights $w_1$, and $w_2$. 

\begin{figure}[t]
	\centering
	\includegraphics[height=1.8in]{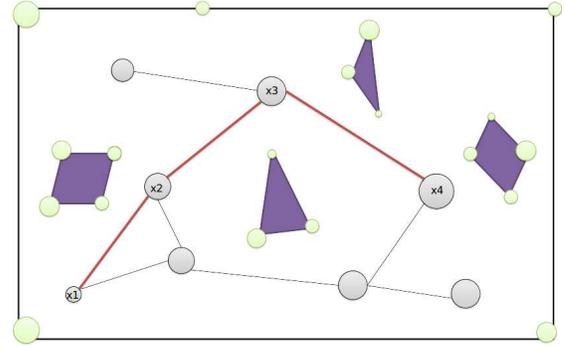}
	\caption{\small Simultaneous planning, localization and mapping problem -- purple polygons represent obstacles, green circles represent features with their size denoting uncertainties in pose estimates. The problem is to find milestones (gray circles) of robot poses, and plan a  trajectory (red line) that minimizes feature uncertainties.} \label{fig splam}
\end{figure}

Recently, graph-based optimization approaches to the SLAM problem have become very popular due to their ability to exploit the naturally sparse connectivity between robot poses and features in the map \cite{Dellaert_IJRR_2006}.
These approaches have proven to have better scalability than the RBPF approaches, which ultimately suffer from particle depletion as the size of the environment grows. Within the graph-based approaches there are two main flavors: pose graphs and feature-based graphs. In the pose-graph approaches, sensor data is used to generate relative transformation constraints between robot poses directly, and an underlying OG map is often required to represent the environment. For example, \cite{Vallve15_ICRA, Valencia_IROS_2012} optimizes the robot trajectory by iteratively computing transformations between laser scans, but still maintains an underlying OG map and plans paths using sample-based approaches such as the probabilistic roadmap or the RRT* algorithm. \cite{Leung08_activeslam} optimizes robot trajectory with features in a structured environment for also maintains an OG for collision check.  Information quantification over the OG map representation carries over known shortcomings of bad scalability and robustness\cite{Grisetti_tutorial_10}. The grid map is also an approximation because the conditional dependencies between the grid cells are discarded.
For example, if there is significant drift in the robot's pose estimate, this uncertainty is not reflected explicitly in the OG map. As a result, a straight corridor will appear curved, but their relative map entropies will be equivalent. 
In addition, OG maps also have large memory footprints.

In a feature-based representation, features are explicitly maintained in the graph and give a layout of the map. However, active-SLAM on feature-based graphs is hard because features do not offer obstacle information, which is crucial for the robot to check path feasibility. 
This paper proposes the first, to our knowledge active SLAM approach that plans robot paths to directly optimize a global feature-based representation without any underlying OG representation. Rather than formulating the problem as area coverage over an OG map \cite{Carlone14_ICRA}, we set it up as entropy reduction over the map features subject to a budget constraint. Since the feature estimates and pose trajectory are necessarily correlated, we can remove the pose uncertainty from the traditional objective function \eqref{eq:active_slam} and directly optimize over the map quality.

When features and robot poses are modeled as joint Gaussian variables, we can directly quantify the information gain of new data on all variables with the same entropy metric.
When the robot moves to a frontier of the mapped area, it can potentially observe new features. The observed features and new features are measured with a unified information metric, therefore we can balance between exploitation and exploration automatically. 
The feature-based graph model is sparser than grid maps, thus scales much better and enables real-time robot state estimation and path planning. Furthermore, it models the environment with dependent features and robot poses rather than with i.i.d. binary cells. Therefore, when the robot returns to a visited place and closes a loop, it can correct long-term drift and propagate the changes to all existing features and poses through the dependencies between variables.

Figure \ref{fig splam} shows an example scenario. The locations of features are marked by green circles and the size of each circle represents its uncertainty. Gray circles represent samples of robot poses, and purple polygons represent obstacles. The planning problem is then to quantify information gains on the samples and find a trajectory connecting the samples that can minimize feature uncertainties.

In summary, there are four primary contributions.
\begin{enumerate}
\item Propose a feature-based topology graph to represent the map of features as well as obstacles in an efficient way. 
\item Develop a feature-focused information metric to quantify uncertainties in the map, in both visited and unvisited places.
\item Present a path planning algorithm using the feature-based topological graph to enable the robot to actively explore with the objective of directly reducing the uncertainty of the map
\item Test the proposed approach in a Gazebo simulated environment, as well as in a real world environment with a turtlebot.
\end{enumerate}
\section{Problem Statement} \label{sec:problem}
Assume that there exists a library of static features that can be uniquely identified as landmarks to localize the robot in the environment, denoted as $\mathbf L=\{ L_1, L_2,\cdots L_M\}$. Notice that the number of features present in the environment could be less than $M$, and is not known a priori. The exact locations of the present features are not known \textit{a priori} either and need to be established by the robot. When moving in the environment, the robot's trajectory is a sequence of poses $ \mathbf X_T=\{X_0,  X_1,\cdots,  X_T\}$, where $X_0$ gives the initial distribution of the robot pose, typically set as the origin with low uncertainty. The robot can obtain two kinds of observations. The odometry $o_t$ is the change between two consecutive poses with probability model $p(o_t|X_t,X_{t-1})$. A feature measurement $z_t^k$ is a measurement between the current pose $X_t$ and feature $y_t$. The corresponding probability model of $z_t^k$ is $p(z_t^k|X_t, L_{y_t^k})$. Denote $z_t = \{z_{t,1},\cdots, z_{t,K_t}\}$.

A factor graph is a sparse representation of the variables. Each node represents either a feature $L_i$ or a robot pose $X_t$. Let $p(\mathbf L)=\prod_{i=1}^M p(L_i)$  denote the prior for features. Each factor is a feature prior $p(L_i)$, an odometry $o_t$  or a feature measurement $z_t^k$.
The joint posterior of $\mathbf X$ and $\mathbf L$ is then the product of priors and likelihood of the observations $\mathbf o=\{o_1,\cdots,o_{T}\}$ and $\mathbf z=\{z_1, \cdots, z_T\}$:
\begin{align}\label{equ:prob_model}
p(\mathbf X, \mathbf L|\mathbf o, \mathbf z) \propto p(\mathbf L) 
\prod_{t=1}^T p(o_t|X_t,X_{t-1}) \prod_{k=1}^{K_t} p(z_t^k|X_t, L_{y_t^k}).
\end{align}
The SLAM problem of jointly inferring the most likely posterior (MAP) feature positions and robot poses can be defined as:
\vspace*{-0.3cm}\begin{align}\label{equ:prob_slam}
\mathbf (\mathbf X^*, \mathbf L^*) = \arg\max_{\mathbf X, \mathbf L} 
p(\mathbf X, \mathbf L|\mathbf o, \mathbf z)
\end{align}
With factor graph representation, \eqref{equ:prob_slam} can be solved by readily available graph-SLAM algorithms/packages such as g2o, iSAM or GTSAM\cite{g2o,Rosen12_isam}.

The problem has traditionally solved by manually operating the robot in the environment to gather a dataset first and then optimize the map in a batch update. In this work, the robot actively plans its own trajectory to incrementally learn the map. Considering that robots are typically constrained in computation/memory, the trajectory should be planned in such a way that resources should be spent on gathering information that is directly related to the robot's goal. The focus of this paper is to incrementally build a map of the environment, therefore information gain is defined as entropy reduction only on variables representing features.

Shannon entropy \cite{Sha48} is a measure of uncertainty in a random variable $x$ thus widely used as information metric. Let $p(x)$ denote the probability distribution of $x$, then entropy $H(x)$ is defined as $H(x)=\sum_{x} p(x)\log p(x)$ for discrete variables and $H(x)=\int p(x)\log p(x)dx$ for continuous variables.

Denote the control command at time $t$ as $u_t$, and let $\mathbf u_T = \{u_1,\cdots, u_T\}$. The active focused planning problem is summarized as follows.

\begin{Problem}\textbf{Active SLAM:}
Design control commands $\mathbf u_T=\{u_1,u_2,\cdots,u_T\}$, such that the robot follows a trajectory that the obtained odometry $\mathbf o=\{o_1,\cdots,o_T\}$ and feature measurements $\mathbf z=\{ z_1, \cdots z_T\}$ can minimize the entropy $H(\cdot)$ over the belief of map features $\mathbf L$:
\begin{equation}
\begin{split}
 \max_{\mathbf u_{T}=\{u_1,\cdots, u_{T}\}}  ~ & H(\mathbf L | \mathbf o, \mathbf z)  \\
\text{s.t.} \quad &  q(\mathbf u_{T})  \leq  c \\
& X_t  = g(X_{t-1}, u_t) \\
& o_t  = X_{t} \ominus X_{t-1} + v, \quad v\sim \mathcal{N}(0,Q)  \\
& z_t^k  = L_{y_t^k} \ominus X_t +w, \quad w\sim \mathcal{N}(0,R) \\
& t=1,\cdots, T
\end{split}
\end{equation}
\label{problem:active}
\end{Problem}
where $q(\cdot)$ is a measure of control cost, in the case of finite time horizon, $q(\mathbf u_T)=T$. Function $X_{t} = g(X_{t-1}, u_t)$ describes the robot dynamics. Function $ o_t  = X_{t} \ominus X_{t-1} + v$ describes the odometry measurement model and $z_t^k  = L_{y_t^k} \ominus X_t +w$ is the feature measurement model. 

Fig.~\ref{fig model_active} presents a graphical model of this problem. $X_t$ represents robot poses, $\mathbf L$ represents environment features. The goal is to design control policies $\mathbf u_T$ to maximize information gain over feature belief $\mathbf L$.

\begin{figure}[t]
\centering
\begin{tikzpicture}
[hidden/.style={circle, draw, minimum size = 1cm}]
	\node[rectangle,  minimum height=0.8cm,minimum width=0.8cm, draw] (plate) at (-0.4,2.5) {$\mathbf L$}; 
	
	\node[hidden] (x0) at (0,0) {$X_0$};
	\node[hidden] (x1) at (2,0) {$X_1$};
	\node[text width=1cm] (dot) at (4,0) {\huge $\cdots$};
	\node[hidden] (xT) at (6,0) {$X_T$};
		
	\node[hidden] (z1) at (2.5,1.5) {$z_1$}; 
	\node[hidden, thick] (u1) at (1,1.5) {$u_1$}; 

	\node[hidden] (zT) at (6.5,1.5) {$z_T$}; 
	\node[hidden, thick] (uT) at (5,1.5) {$u_T$}; 

	\draw[thick,->] (x0) --(x1);
	\draw[thick,->] (x1) --(dot);
	\draw[thick,->] (dot) --(xT);
	
	\draw[thick,->] (plate) -|(z1.north);
	\draw[thick,->] (x1) --(z1);			
	\draw[thick,->] (u1) --(x1);
	
	\draw[thick,->] (plate) -|(zT.north);
	\draw[thick,->] (xT) --(zT);			
	\draw[thick,->] (uT) --(xT);

	
	\end{tikzpicture}
\caption{\small \textbf{Active Focused Planning}. The robot uses landmark measurements $z_t$ and odometry to design a control policy $u_{1:T}$ to maximize information gain over environmental features $\mathbf L$.}\label{fig model_active}
\vspace{-0.2in}
\end{figure}
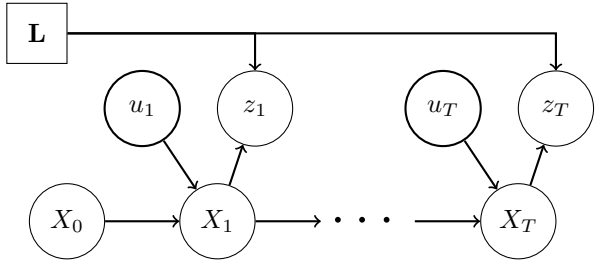

%
%
%

\section{Method} \label{sec:method}
\subsection{Topology feature graph}
One important reason that the use of grid-map representation has been a popular choice for active-SLAM is that a grid-based map contains all the necessary information for path planning. A feature-based representation, although much sparser, lacks information about free/occupied space and the topology of the environment. Consequently, planning paths over a traditional feature-based representation is ill-posed. To overcome this, we propose to store additional information with each feature that allows us to generate a full, yet sparse, representation of the environment over which we can then plan paths.

In this paper, we assume the robot is a ground robot that operates in 2D space\footnote{Extension to 3D scenarios can be achieved by triagularizing obstacle surfaces and is left to future work}. Relying on the fact that features are usually on the surface or corner of obstacles, we propose the \textit{Topological Feature Graph} (TFG) representation. A TFG is a graph $\mathcal{G} = \{L, E\}$, with its vertices representing features and edges representing obstacles. More specifically, if two features are connected by an edge, then these two features belongs to the same flat obstacle surface and the edge is not traversable\footnote{Features can be extended to objects that have sizes, in which case obstacles would be represented by both objects represented by vertices and surfaces represented by edges}.

These edges can be learned from either a depth image, a laser scan or even sequences of images \cite{Pillai_RSS_15}. The robot first segments the depth map or laser scan into several components representing different obstacle surfaces, then checks if two features detected belong to the same component. If so, the robot creates an edge between these two features. This idea is illustrated in Figure \ref{fig:TFG}.

Compared to the grid map representation, the TFG offers several advantages in structured environments. First it requires many fewer variables to represent the environment, and thus provides significant memory savings.
Second, the map complexity can easily adapt to various complexities in the environment. Instead of using equal sized cells at all places, a TFG can model more features in cluttered/narrow spaces and less features in wider/simpler spaces. Third, if new loop closures are detected and drifts of some subgraphs are corrected, the obstacles will be corrected with the feature positions: the robot does not have to relearn the occupancy of the associated space. And finally, this representation has a closed-form collision check for robot path planning rather than sampling-based methods, leading to significant computation savings in path planning.

\subsection{Sequential Planning}\label{sec:observation}
Recall that our goal is to plan robot controls that gain maximal information from the environment as formulated in Problem \ref{problem:active}. Notice that solving Problem \ref{problem:active} in batch is hard in general, because at any time $t$, observations beyond $t$ are not available, thus planning controls $u_t,\cdots u_T$ will require modeling future observations and taking into account all possible outcomes, which is typically intractable.

To solve this problem in a tractable manner, a widely used technique is to split Problem~\ref{problem:active} into $T$ stages, optimize a goal point at each stage \cite{AISTATS07_WilliamsFW, NewStadt_tsp_15}. For each stage, a separate path planner can be used to generate controls.

Let $p(\mathbf L)$ denote a prior of the landmarks. At stage $t$, the observation history $o_{1:t}$ and $z_{1:t}$ can be summarized in a posterior distribution of $\mathbf L, \mathbf X$ at time $t$. 
Denote the maximal posterior(MAP) values of $\mathbf X$ and $\mathbf L$ as $\mathbf X^*_t$ and $\mathbf L^*_t$, they can be obtained by standard SLAM solvers:
\begin{align}\label{equ:posterior}
\mathbf  X^*_t, \mathbf L^*_t
= & \argmax~ p( \mathbf X, \mathbf L|\mathbf o_t, \mathbf z_t) \\
=& \argmax~ p(\mathbf L)\prod_{\tau=1}^t p(o_\tau|X_\tau,X_{\tau-1} ) p(z_\tau|X_\tau,\mathbf L)  \notag
\end{align}
Use the standard procedure of Laplacian approximation of $\mathbf X$ and $\mathbf L$: a Gaussian distribution with the mean being its MAP values $ (\mathbf X^*_t, \mathbf L^*_t)$, and the information matrix $\Lambda$ being the second moment:
\begin{align}\label{equ:laplacian}
& \mathbf X, \mathbf L |\mathbf o_t,\mathbf z_t \sim \mathcal{N}(\mathbf X^*_t, \mathbf L^*_t; \Lambda^{-1}) \\[1em]
& \Lambda = \frac{\partial^2 \log p(\mathbf L)}{\partial(\mathbf X,\mathbf L)^2}
+ \sum_{\tau=1}^t \frac{\partial^2 \log p(o_\tau)}{\partial(\mathbf X,\mathbf L)^2}
+ \sum_{i=1}^{M} \frac{\partial^2 \log p(z_{\tau})}{\partial(\mathbf X,\mathbf L)^2} \notag \\
&= \matblock{{cc}\Lambda_{f} & \Lambda_{fr} \\ \Lambda_{rf} & \Lambda_{r}}
\end{align}
where $\Lambda_f$ corresponding to landmarks and $\Lambda_r$ corresponds to robot poses.
Laplacian approximation gives close-form solutions for entropy. The marginal information matrix for landmarks is $\Lambda_{\mathbf L} = \Lambda_f - \Lambda_{fr}\Lambda_{r}^{-1}\Lambda_{rf}$, and the entropy is \begin{align}
H(\mathbf L|\mathbf o_t,\mathbf z_t) = -\frac{1}{2}\log|\Lambda_{\mathbf L}| + \text{constant}
\end{align}
Further with the associated connectivity edges between landmarks, we obtain the TFG at time $t$. Denote it as $TFG_t$, which summarizes the information the robot has about the environment up until time $t$.

The Laplacian approximation simplifies the information quantification, but directly optimizing over controls $u_t$ is still very difficult. Control inputs $u_t$ affect robot paths though robot dynamics, and optimization under both robot dynamic constraints and obstacle constraints would be computationally prohibitive. As such, the problem is further simplified here by planning a trajectory for the robot first, then using a separate path-following controller to drive the robot along the planned trajectory. In this way, controller design is decoupled from path planning. 
\begin{Problem}\label{problem:incremental}\textbf{Path Planning for Active SLAM}
At stage $t$, given prior topological feature graph $TFG_t$, find a path $\widehat X_t \cdots, \widehat X_\tau, \cdots, \widehat X_{t+1}$, such that the posterior entropy on landmarks $\mathbf L$ is minimized:
\begin{align}
\min_{\widehat X_t, \cdots, \widehat X_{t+1}} & \quad H(\mathbf L|TFG_t, \hat{\mathbf{o}}_{t+1}, \hat{\mathbf{z}}_{t+1}) \notag \\
\text{s.t.} \quad 
& \widehat X_\tau  = g(\widehat X_{\tau-1}, u_\tau) \notag \\
& \hat o_\tau  = \widehat X_{\tau} \ominus \widehat X_{\tau-1} + v, \quad v\sim \mathcal{N}(0,Q) \notag \\
& \hat z_\tau^k  = L_{y_\tau^K} \ominus \widehat X_\tau + w, \quad w\sim \mathcal{N}(0,R) \notag \\
& \tau = t,\cdots, t+1
\end{align}
where $\mathbf{\hat o}_\tau = \{\hat o_t, \cdots, \hat o_\tau, \cdots, \hat o_{t+1}  \}$ represents the odometry obtained along the trajectory. And  $\mathbf{\hat z}_{t+1}=\{ \hat z_t, \cdots, \hat z_\tau,\cdots \hat z_{t+1} \}$ represents landmark measurements obtained along the trajectory. 
\end{Problem}
Given the path $\widehat X_t, \cdots, \widehat X_{t+1}$ and the partial TFG at time $t$, a separate path-following controller could be used to drive the robot along the trajectory. In this way, path planning and control are decoupled from the active SLAM problem, and we gain performance in computation and speed.

\subsection{Expected Information Gain}
Quantifying the exact information gain from $\widehat X_t$ to $\widehat X_{t+1}$ in Problem \ref{problem:incremental} is challenging because it involves discretizing the trajectory from $X_t$ to $\widehat X_{t+1}$ into a sequence of robot poses $\widehat X_\tau$, then computing the information gain from measurements at each pose. Information gain of measurements on later poses will depend on earlier poses along the path. Therefore, the complexity will grow exponentially with the path length. 
To solve the information quantification problem in real-time, we only plan a goal point for the robot, design the robot to stabilize its pose at the goal point, rotate in-place to obtain accurate observations of the local environment, and compute information gain only on these locally observable landmarks at the goal point. The observation would be some layout of a subset of the local landmarks. As shown in  Figure \ref{fig:observation point}, gray balls denote observation points, and the blue circle indicate the set of landmarks it can observe at those observation points. 

\begin{Problem}\label{prob:incre_goal}\textbf{Goal Planning for  Active SLAM}
At stage $t$, given prior topological feature graph $TFG_t$, find the next goal point $\widehat X_{t+1}$ such that the entropy on landmarks $\mathbf L$ is minimized:
\begin{align}
\min_{\widehat X_{t+1}} & \quad   H(\mathbf L|TFG_t, \hat z_{t+1}) \notag \\
\text{s.t.} \quad 
& \hat z_{t+1} = h(\hat X_{t+1}, \mathbf{L})
\end{align}
\end{Problem}

Goal points also provide a way to segment the overall map into local maps and sparsify the underlying SLAM factor graph: the robot accurately maps the environment at goal points, thus measurements between two goal points contains less information compared to those at goal points. Therefore, landmark measurements along the path are only used to localize the robot, but are not used to update landmark estimates. This may cause some loss of information. However, with this simplification, we can marginalize out robot poses between two observations points, and the SLAM factor graph will become a joint graph of partial graphs at goal points. In  this way, the complexity of the SLAM factor graph only scales with the number of observation points and not the number of robot poses. 

Furthermore, paths are generated with respect to the current estimate of landmark locations. If measurements along a path are used to update landmark estimates, new loop closures may cause shifts in landmark locations. The old path may become invalid and the robot may run into obstacles. Leaving out measurements along the path also avoids this potential failure.
\begin{figure}[t]
	\begin{subfigure}[b]{0.48\columnwidth}
		\centering \includegraphics[trim=80 0 80 0,clip,width= 0.7\columnwidth]{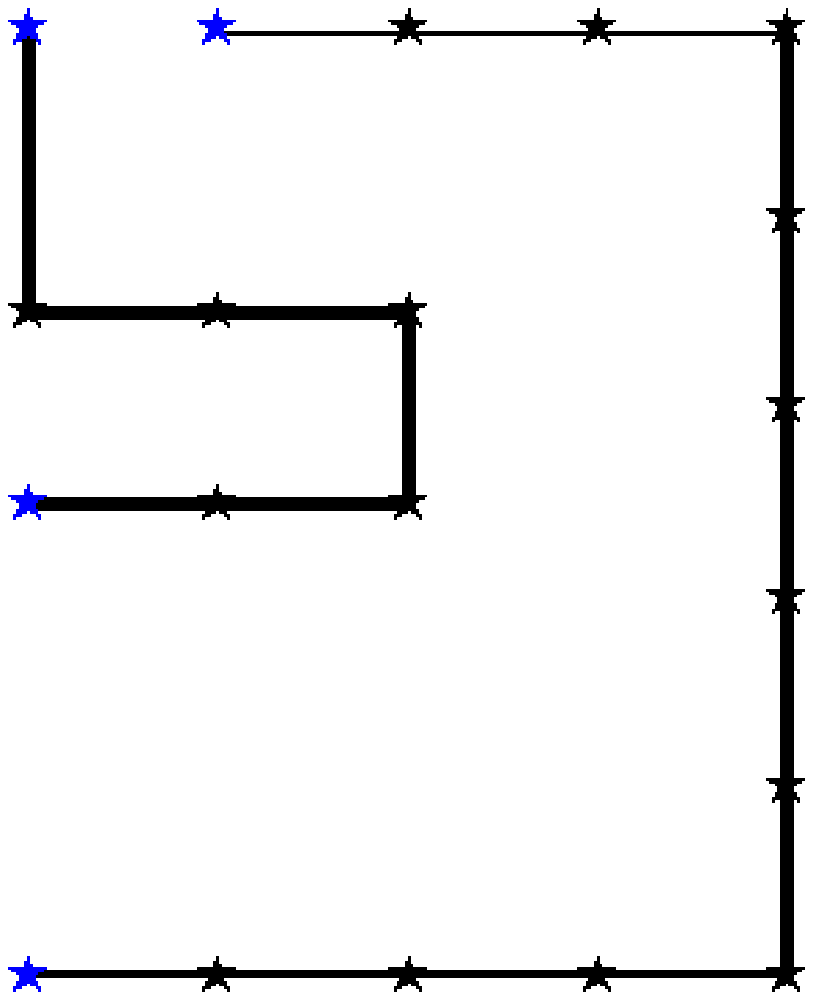}	
		\caption{\footnotesize Topology Feature Graph(TFG)}\label{fig:TFG}		
	\end{subfigure}	
	\begin{subfigure}[b]{0.48\columnwidth}
		\centering \includegraphics[trim=80 0 80 0,clip,width = 0.7\columnwidth]{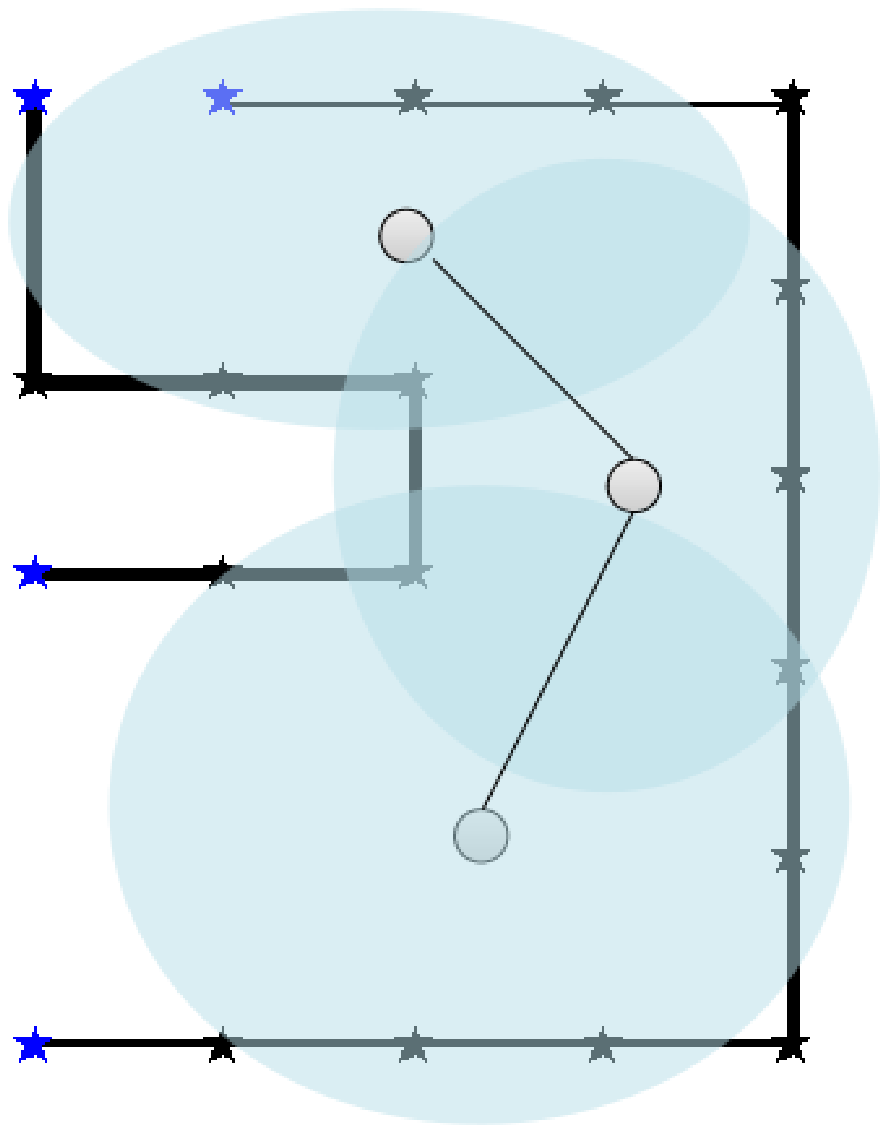}
		\caption{\footnotesize Goal points}\label{fig:observation point}
	\end{subfigure}
\caption{\small Topological Feature Graph (TFG) and goal points. Vertices (stars) represent features, edges (black lines) represent obstacle surfaces, blue stars represent features at a frontier, gray balls represent goal points.  Blue regions illustrate local observable features.}
\vspace{-0.2in}
\end{figure}

Notice that $\widehat X_{t+1}$ is in continuous $\mathbb{R}^2$ space. Different $\widehat X_{t+1}$ would give different combinations of observable landmarks, thus solving problem \ref{prob:incre_goal} exactly would be hard. Instead, we use a random sampling approach. Given $TFG_t$, a location is \textit{reachable} if it can be observed from some previous goal location. The planner samples locations in the robot's reachable space, computes entropy reduction for each goal point, then selects the next goal point as the one that gives maximal entropy reduction.  

The maximal entropy reduction problem can be stated as follows:
\begin{align}\label{equ:entropy_reduction}
\widehat X_{t+1} =& \arg\max_{X_{t+1}} \quad \Delta H(\mathbf L|X_{t+1}, TFG_t) \\
	 =& \arg\max_{X_{t+1}}  \quad H(\mathbf L|TFG_t)- H(\mathbf L| X_{t+1}, TFG_t)\notag 
\end{align}

\begin{thm}\label{thm:explore_exploit}
Set prior covariance for unknown landmarks in such a way that it is much larger than covariance of observed landmarks. Given topological landmark graph $TFG_t$, the entropy reduction at goal location $X_{t+1}$ can be approximated by the sum of entropy reduction on local observable landmark $dH_o$, and of new landmarks $dH_u$
\begin{align}
	\Delta H(X_{t+1}|X_{t+1}, TFG_t) \approx \Delta H_o + \Delta H_u
\end{align}
where $\Delta H_u=n_x\log |I+\sigma_ua_u|$, $n_x$ is the number of new landmarks $n_x$, $\log |I+\sigma_ua_u|$ the expected information gain on an unknown landmarks.
\end{thm}
\begin{proof}
For simplicity, the subscript $t$ is dropped in the following, but it should be noted that this analysis is based on $TFG_t$. 

The information matrix can be written into two parts that corresponds to observed landmarks $\Lambda_f$ or robot poses $\Lambda_r$
\begin{align}
\Lambda & = \matblock{{cc} \Lambda_f & \Lambda_{fr} \\ \Lambda_{rf} & \Lambda_{r}}  \notag 
\end{align}

The robot task here is to map the landmark, therefore we only look at the marginal information matrix on landmarks:
\begin{align}
\Lambda^t_{\mathbf L} & = \Lambda_{f}-\Lambda_{fr}\Lambda_r^{-1}\Lambda_{rf} = \matblock{{cc} \Lambda_o & 0\\ 0 & \Lambda_u}  \notag 
\end{align}
where $\Lambda_o$ corresponds to landmarks observed at least once, and $\Lambda_u$ corresponds to landmarks that have not been observed yet.
At goal point $\widehat X_{t+1}$, denote $\hat z_{t+1}$ are the expected new landmark measurements, then the new joint likelihood becomes:
\begin{align}\label{equ:posterior2}
& p(TFG_t,\hat z_{t+1}; \widehat X_{t+1},\mathbf X, \mathbf L) \notag \\
\sim &  p(TFG_t; \mathbf X, \mathbf L)  p(\hat z_{t+1}|\widehat X_{t+1}, TFG_t)
\end{align}
The corresponding factor graph is the factor graph at $t$ plus new landmark measurements $p(\hat z_{t+1}|\widehat X_{t+1},TFG_t)$. Using the same ML values $X_t^*$, $\mathbf L^*_t$ in \eqref{equ:prob_slam}, the new information matrix $\Lambda^{t+1}_{\mathbf L}$ would be the original information matrix $\Lambda_{\mathbf L}^t$, plus some new terms coming from factors $p(\hat z_{t+1}|\widehat X_{t+1},TFG_t)$:
\begin{align}
\Lambda^{t+1}_{\mathbf L} = \matblock{{ccc} \Lambda_o & 0 & 0\\
	 0 & \Lambda_u & 0 \\ 0 & 0 & 0}
+\matblock{{ccc} A_o & 0 & H_o\\
	0 & A_u & H_u \\
	H_o^T & H_u^T & B}
\end{align}
where
\begin{align}
B &= B_o + B_u \notag \\
 \matblock{{cc}A_o & H_o \\H_o^T & B_o} &= \frac{\partial^2 p(z_{t+1}|\widehat X_{t+1},TFG_t)}{\partial(\mathbf L_o, \widehat X_{t+1})^2}  \\ 
 \matblock{{cc}A_u & H_u \\H_u^T & B_u} &= \frac{\partial^2 p(z_{t+1}|\widehat X_{t+1},TFG_t)}{\partial(\mathbf L_u, \widehat X_{t+1})^2} \notag 
\end{align}
The marginal information matrix on landmarks can be computed from the Schur complement:
\begin{align}
\Lambda^{t+1}_{\mathbf L} &= \matblock{{cc} \Lambda_o+A_o & 0 \\	0 & \Lambda_u+A_u} - 
	H B^{-1} H^T \notag \\
H &= \matblock{{c}H_o\\H_u}
\end{align}
Note that elements in $A$ and $H$ are 0 if the corresponding landmark is not observable at observation point $\hat X_{t+1}$.
The incremental change in the information objective $H(\cdot)$ is: 
\begin{align}
\Delta H &= -\log|\Lambda^t_{\mathbf L}| + \log|\Lambda^{t+1}_{\mathbf L}| \notag \\
& = \log \left|  
	\matblock{{cc}\Lambda_o+A_o & 0 \\0&\Lambda_u+ A_u} - H B^{-1} H^T  \right|  \notag \\
&-\log \left| \matblock{{cc}\Lambda_o & 0 \\0&\Lambda_u} \right| 
\end{align}
Take the inverse of the matrix in second term, combine it with the first term, then use 
$\Lambda^{-1}_o = \Sigma_o$,  $\Lambda^{-1}_u=\Sigma_u$ to obtain
\begin{align}
& \Delta H \notag \\
& = \log \left|  
\matblock{{cc} I+\Sigma_oA_o & 0 \\0& I+\Sigma_uA_u} -
\matblock{{cc}\Sigma_o & 0 \\0&\Sigma_u}
H B^{-1} H^T  \right|  \notag
\end{align}
Extract the first term to get
\begin{align}
\Delta H & =  \log \left|  I +
\matblock{{cc}\Sigma_oA_o & 0 \\0&\Sigma_uA_u} \right| \notag \\
& + \log \left|I - \matblock{{cc}I+\Sigma_oA_o & 0 \\0&I+\Sigma_uA_u}^{-1}
H B^{-1} H^T  \right| \notag
\end{align}
Apply $|I-BA| = |I-AB|$ on the second term
\begin{align}
& =  \log \left|  I + \matblock{{cc}\Sigma_oA_o & 0 \\0&\Sigma_uA_u} \right|  \notag \\
& + \log \left|I - B^{-1} H^T \matblock{{cc} (I+\Sigma_oA_o)^{-1} & 0 \\0 &  (I+\Sigma_uA_u)^{-1}} H \right| \notag \\
& =  \log | I+\Sigma_oA_o|+ \log|I+\Sigma_uA_u|   \\
& + \log \left|I - B^{-1}H_o^T(I+\Sigma_oA_o)^{-1}H_o -  B^{-1}H_u^T (I+\Sigma_uA_u)^{-1} H_u \right| \notag 
\end{align}
\noindent When a landmark has not been previously observed, the prior covariance $\Sigma_u$ is typically large, therefore $H_u^T(I+\Sigma_uA_u)^{-1}H_u$ is small compared to $H_o^T(I+\Sigma_oA_o)^{-1}H_o$. Furthermore, notice that when the prior $\Sigma_u$ and information delta $A_u$ are block diagonal, with each block representing a landmark, $ \log|I+\Sigma_uA_u|  = n_x \log|I+\sigma_u a_u|$, and we have the following approximation: 
\begin{align}\label{equ:info_gain_approx}
\Delta H \approx & \log | I+\Sigma_oA_o|+ \log \left|I - B^{-1}H_o^T(I+\Sigma_oA_o)^{-1}H_o\right|  \notag \\ 
& + n_x \log|I+\sigma_u a_u|  \notag \\
=& \log|I +\Sigma_oA_o - H_oB^{-1}H_o| + n_x \log|I+\sigma_u a_u|\notag \\
=& \Delta H_o + \Delta H_u
\end{align}
where $\Delta H_o= \log|I +\Sigma_oA_o - H_oB^{-1}H_o|$ is the information gain obtained by having new measurements on observed landmarks,  $\Delta H_u=n_x \log|I+\sigma_u a_u| $ is information obtained by having new measurements on previously unobserved landmarks, $n_x$ is the number of new landmarks observed, and $\sigma_u$ and $a_u$ are the variance and information gain of a single new landmark.
\end{proof}
Theorem \ref{thm:explore_exploit} indicates that the information gain on a goal point can be split into two parts: the first part $\Delta H_o$ is the information gain obtained by re-observing and improving known landmarks, and $\Delta H_u$ is the information gain from exploring new landmarks. In our experiments, $n_x$ is computed by using a predefined landmark density in the environment multiplied by the size of a frontier at observation point $X_{t+1}$. 

\subsection{Frontier Detection}
In order to detect frontiers, we track how each landmark is connected to its neighbors. A landmark borders a frontier if at least one side of it is not connected to any neighbors. As shown in  Figure \ref{fig:TFG}, the blue stars represent landmarks at frontiers. At each sample location, the size of frontier is computed as following:
\begin{enumerate}
\item Compute landmarks the robot expects to observe
\item Sort the landmarks according to their orientation relative to the robot
\item If two consecutive landmark are not connected to any neighbors, they represent a frontier.
\end{enumerate}
With this frontier detection approach, we have a uniform information metric for both observed landmarks and unobserved landmarks at frontiers. Thus our approach gives a natural balance between exploration and exploitation: if there are large frontiers offering the potential to discover many new landmarks, the robot will pick observation points to explore frontiers. If there are only small frontiers or none at all, the robot might go to visited places to improve existing landmarks estimates.

\subsection{Path Planning}
In Section \ref{sec:observation}, we obtained a set of collision-free samples, therefore the path planner will only compute connectivity and cost between these samples, and form a probabilistic roadmap (PRM). The trajectory to the next best observation point is generated by computing a minimum cost path on a PRM. The cost of an edge between two sample points involves two factors:
\begin{itemize}
	\item The length of the link, which reflects the distance that needs to be traveled and thus the control costs.
	\item Collision penalty. 
\end{itemize}
Computing the exact collision probability of a given path is a computationally expensive procedure. However, exploiting the fact that a collision check for a point using a TFG representation can be carried out analytically enables expensive methods such as Monte Carlo methods for real-time collision evaluation. 
in this work, assuming Gaussian localization uncertainty, we rely on very efficient approximate methods to compute a measure of risk instead of the exact collision probability. 

Denote $x_o$ as the closest obstacle point. Then $||x-x_o||^2$ represents the squared distance to the closest point and reflects the chance of collision. Thus we use $||x-x_o||^2$ as an additive penalty in the edge cost in path planning. With our TFG representation, computing $||x-x_o||^2$ reduces to computing point-line and line-line distances, which can be achieved trivially.

One of the key benefits of relying on the TFG for path planning is the analytic computation of collisions. In other words, since TFG is composed of set of lines, one can analytically verify a given point is in the obstacle region or not by checking TFG lines around the robot. Such a fast collision check enables accurate methods such as Monte Carlo to evaluate collision probability along the path. We rely on a chance constraint formulation (similar to \cite{Pavone_RAM_2009, Luders14_PhD, Blackmore_TRO_2011}) to compute paths that satisfy $ \Pr(\text{path}\in \text{Obstacle})<\delta $. If one relaxes this constraint to $ \Pr(x_k\in \text{Obstacle}) < \delta \forall k,~~ x_k\sim\mathcal{N}(\hat{x}_{k},P_{k}) $
\begin{align}\label{equ:cc-planning}
\Pr(x_k\in \text{TFG edge}) < \delta~~~ \forall k,~~~ x_k\sim\mathcal{N}(\hat{x}_{k},P_{k})
\end{align}
where, $ x_k $ is the $ k $-th point on the trajectory.

\section{Experiments} \label{sec:experiments}
\subsection{Information Measures}\label{sec:info_balance}
We first illustrate how the proposed framework can balance exploration and exploitation. Figure \ref{fig:explore_exploit} shows an example scenario: black lines represent obstacles, stars represent features with blue stars bordering frontiers. Circles are samples in the free space with color representing their information gain: red is high gain and blue is low. Figure \ref{fig:info_exploit} displays the information gain on observed features: the total information gain is largest at samples that can potentially observe the greatest number of features. On the other hand, Figure \ref{fig:info_explore} shows the information gain on new features. The samples closer to frontiers will have a chance to observe new features, thus they have higher exploration gains than samples further from frontiers. Assuming a fixed new feature density, larger frontiers offer the potential to observe more new features and therefore nearby samples have greater exploration information gain. Figure \ref{fig:explore_exploit} shows the total exploration and exploitation information gain.

\begin{figure}[ht]
\centering
\begin{subfigure}[b]{0.3\columnwidth}
	\centering	\includegraphics[height=1.3in, trim = 100 0 100 0,clip]{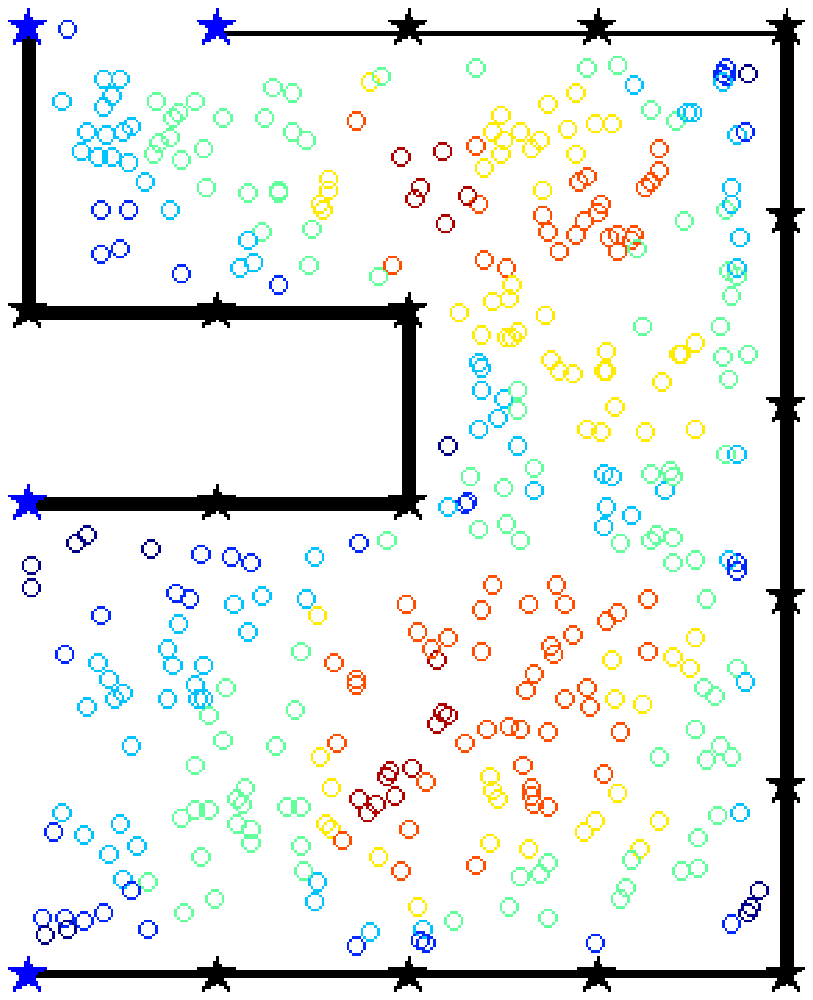}
	\caption{\footnotesize Exploitition} \label{fig:info_exploit}
\end{subfigure}
\begin{subfigure}[b]{0.3\columnwidth}
	\centering	\includegraphics[height=1.3in, trim = 100 0 100 0,clip]{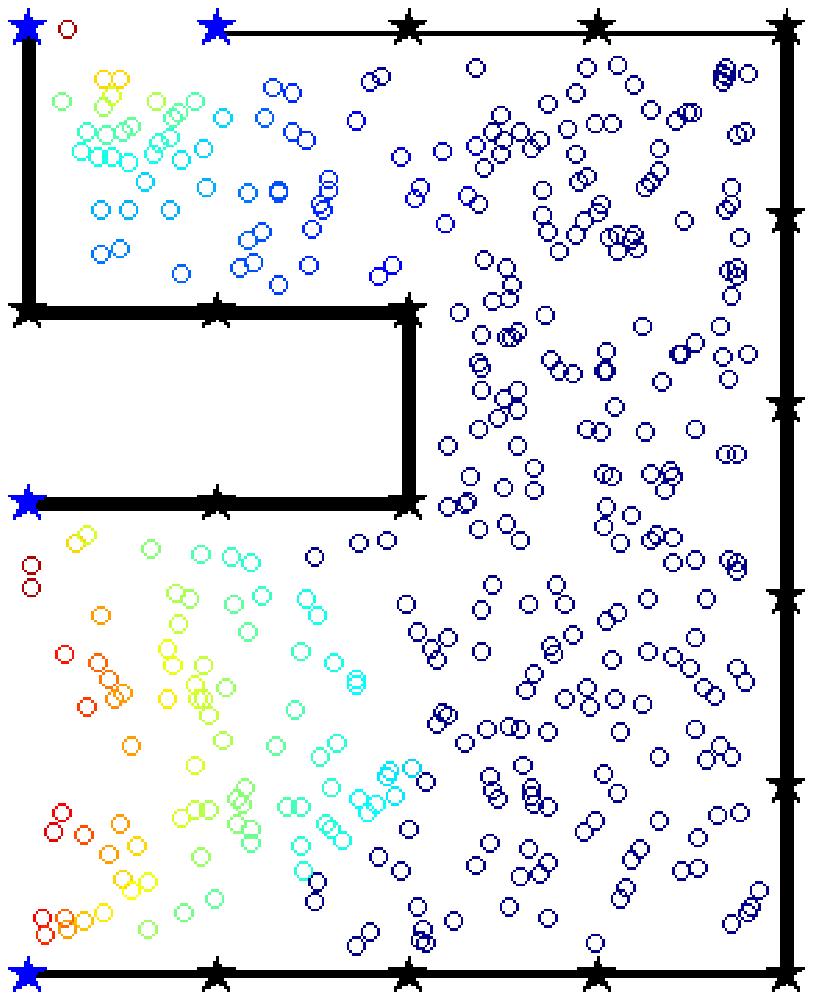}
	\caption{\footnotesize Exploration} \label{fig:info_explore}
\end{subfigure}	
\begin{subfigure}[b]{0.3\columnwidth}
	\centering	\includegraphics[height=1.3in, trim = 100 0 100 0,clip]{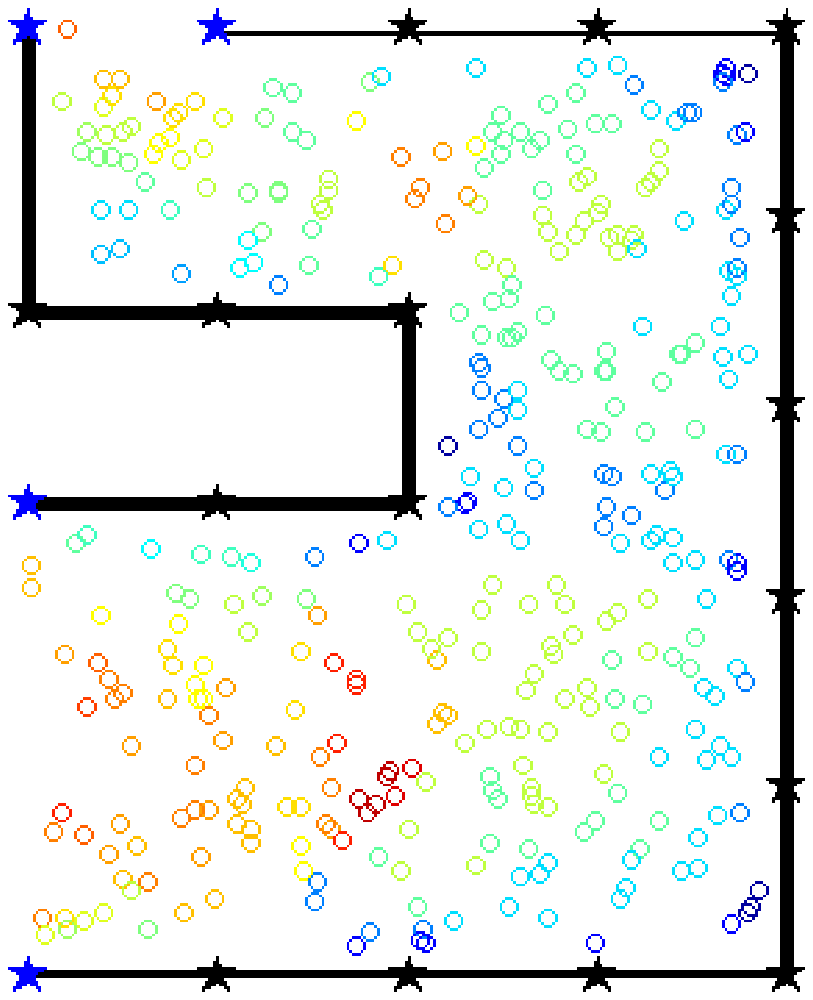}
	\caption{\footnotesize Total}
\end{subfigure}
\caption{\small Information gain. Black lines (obstacles) and stars (features) comprise the TFG. Blue stars indicate frontier features. Circle color represents information gain on samples.}\label{fig:explore_exploit}
\end{figure}

Summing both the exploitation and exploration information, Figure \ref{fig:info_total} displays the total information gain under high/medium/low prior variance on new features. When prior variance on new features is high, observing a new feature will give large information gains, thus the exploration term dominates the exploitation term, and the robot prefers sample points at frontiers. On the other hand, if the prior variance is set to be low, observing new features does not add much information, and the robot will prefer to revisit places with observed features and improve its estimate of their positions. 

\begin{figure}[t]
\centering
\begin{subfigure}[b]{0.3\columnwidth}
	\centering	\includegraphics[height=1.5in, trim = 100 0 100 0,clip]{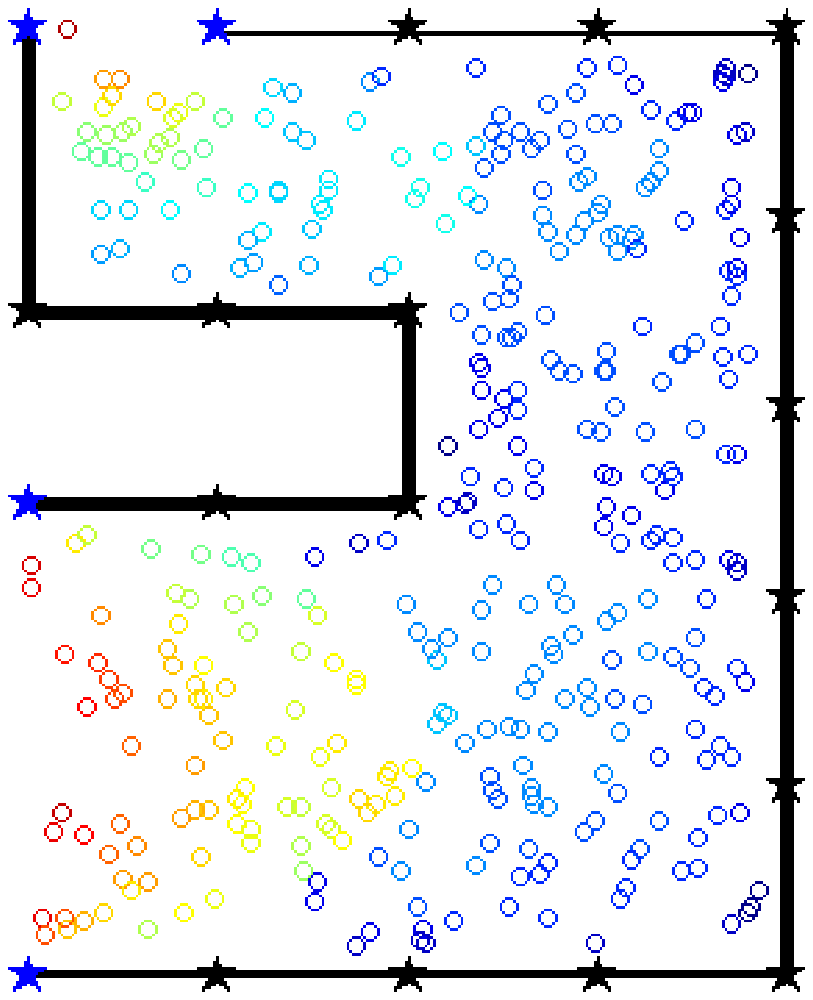}
	\caption{\footnotesize High}
\end{subfigure}
\begin{subfigure}[b]{0.3\columnwidth}	\centering	\includegraphics[height=1.5in, trim = 100 0 100 0,clip]{info_explorebonus_medium}
	\caption{\footnotesize Medium}
\end{subfigure}	
\begin{subfigure}[b]{0.3\columnwidth}
	\centering	\includegraphics[height=1.5in, trim = 100 0 100 0,clip]{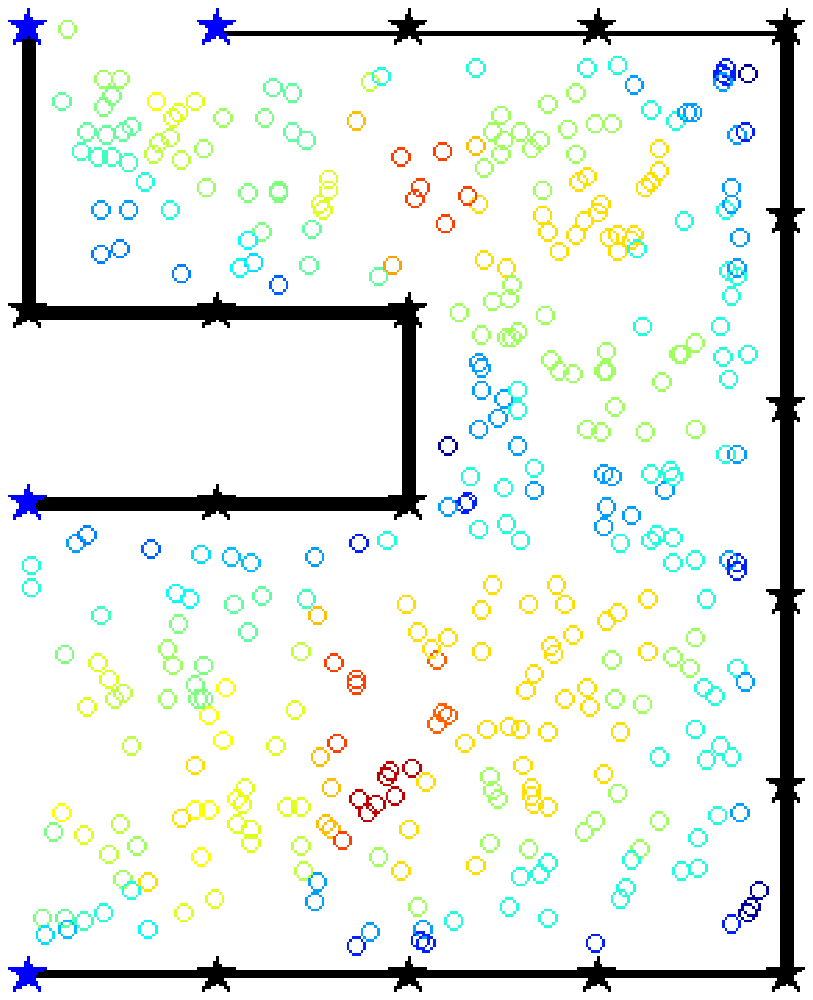}
	\caption{\footnotesize Low}
\end{subfigure}	
\caption{\small Total information gain with varying unseen feature density. When the robot expects to see many features beyond frontiers, information gain at frontiers is high. Otherwise, the robot prefers spots that can observe the most features in visited places.}\label{fig:info_total}
\end{figure}

\subsection{Simulation}
We compared our framework with a nearest-frontier exploration algorithm \cite{Yamauchi97} using the Gazebo simulator. The frontier exploration simulation used the popular GMapping \cite{grisetti2007improved} system for localization and mapping and used wavefront frontier detection \cite{keidar2011fast} to identify frontiers. 

The simulated TurtleBot receives noisy odometry, laser scans, and feature measurements. Table \ref{tab:sim} contains the simulation parameters. Figure \ref{fig:gazebo_map} displays a screenshot of the simulated environment (simulated april tags are spaced roughly one meter apart along the walls). 

Figures \ref{fig:frontier_map} and \ref{fig:active_slam_map} display the maps generated by frontier exploration and TFG active SLAM respectively over one run. Note that there is obvious distortion along the hallways, and the boundary of some obstacles in the center are blurred as well. On the other hand, TFG active SLAM was able to close loops on features and thus maintain the shape of the building in its map.

Figure \ref{fig:error} compares the robot pose error of nearest-frontier and TFG active SLAM over 4 runs. The solid lines represent error mean and shades represent error range. TFG active SLAM consistently has significantly less error in its robot pose estimates, especially in position. Figure \ref{fig:exploration_comparison} compares the map coverage with time spent exploring. In TFG active SLAM, the robot balances exploration with loop closing and is thus slightly slower in covering the whole space when compared with greedy frontier exploration. Table \ref{tab:sim_compare} compares algorithm performance. Although TFG active SLAM takes slightly longer to explore the environment, it uses orders of magnitude fewer variables to represent the world, which leads to memory savings. Frontier exploration also updates particles and the grid map continuously while TFG exploration only updates its map at goal points, requiring only light computation throughout most of its operation.

\begin{table}[h]
\centering
\caption{Simulation parameters}\label{tab:sim}
\begin{tabular}{c|c}
\hline
size of environment & 46m$\times$22m \\
No. of landmarks & 274 \\
sensor range & 10m \\
field of view & 124 degrees\\
particles for gmapping & 100 \\
rate for gmapping update & 0.33Hz \\
rate  for landmark measurements & 10Hz\\
\hline
\end{tabular}
\end{table}

\begin{table}[h]
\centering
\caption{Simulation Performance Comparison} \label{tab:sim_compare}
\begin{tabular}[h]{c|cc}
\hline	& TFG Active SLAM & grid map frontier \\
\hline
No. of variables & 274 & 800000 \\ 
CPU idle time & 75\% & 0\%  \\ 
running time (s) & 2433 $\pm$ 546 & 2293 $\pm$ 375 \\ 
position error (m) & 0.147 $\pm$  0.115 & 5.26 $\pm$ 3.53 \\
orientation error (rad) & 0.0217 $\pm$ 0.016 & 0.0213 $\pm$ 0.0165 \\ \hline
\end{tabular}
\end{table}

\begin{figure*}[th]
\centering
\begin{subfigure}[b]{0.3\textwidth}
	\centering\includegraphics[height=1.4in]{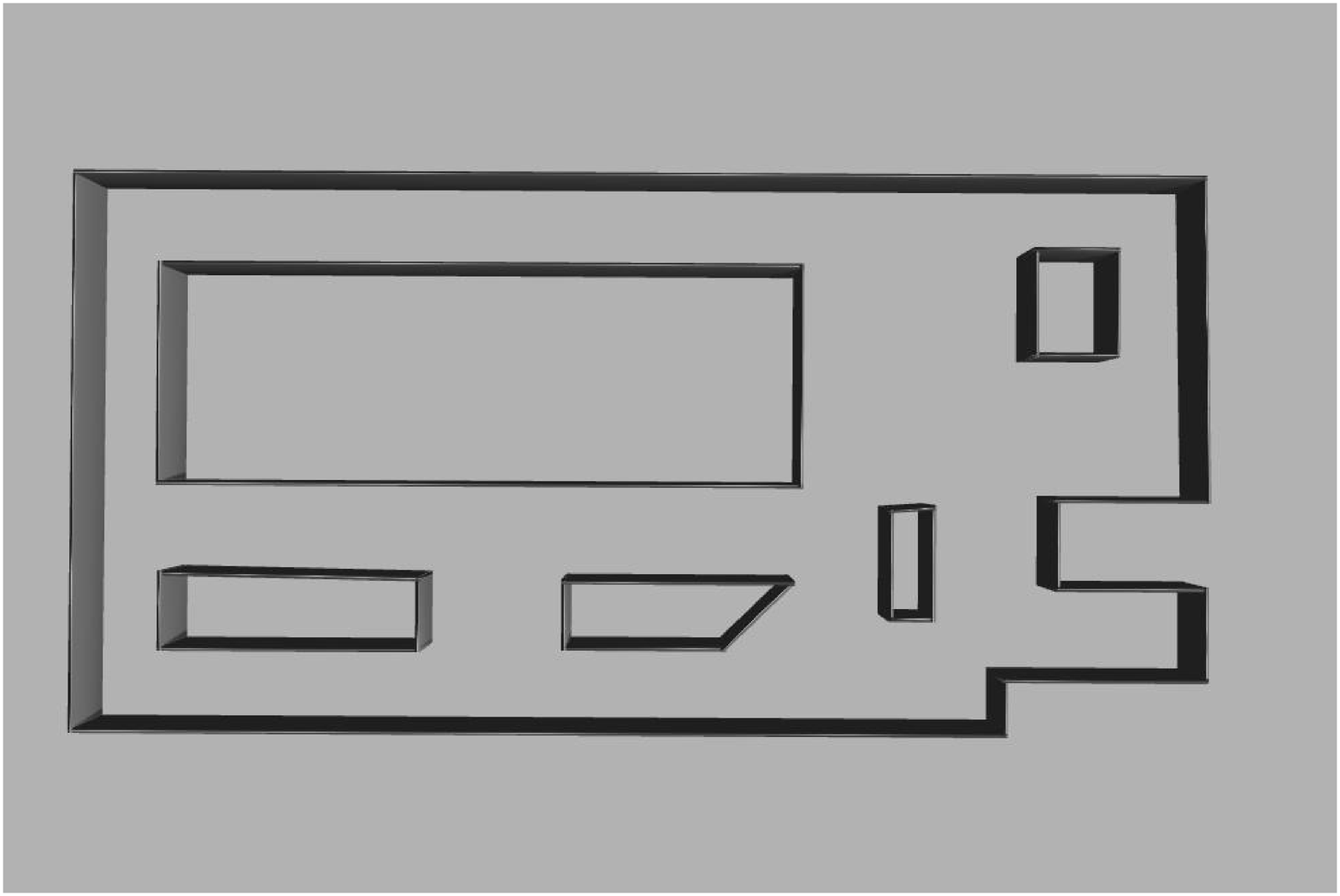}
	\caption{\small Gazebo simulation environment} \label{fig:gazebo_map}
\end{subfigure}
\begin{subfigure}[b]{0.3\textwidth}
	\centering
	\includegraphics[height=1.4in]{frontier_grid_map}
	\caption{\small frontier exploration with grid map} \label{fig:frontier_map}
\end{subfigure}	
\begin{subfigure}[b]{0.33\textwidth}
	\centering
	\includegraphics[height=1.4in]{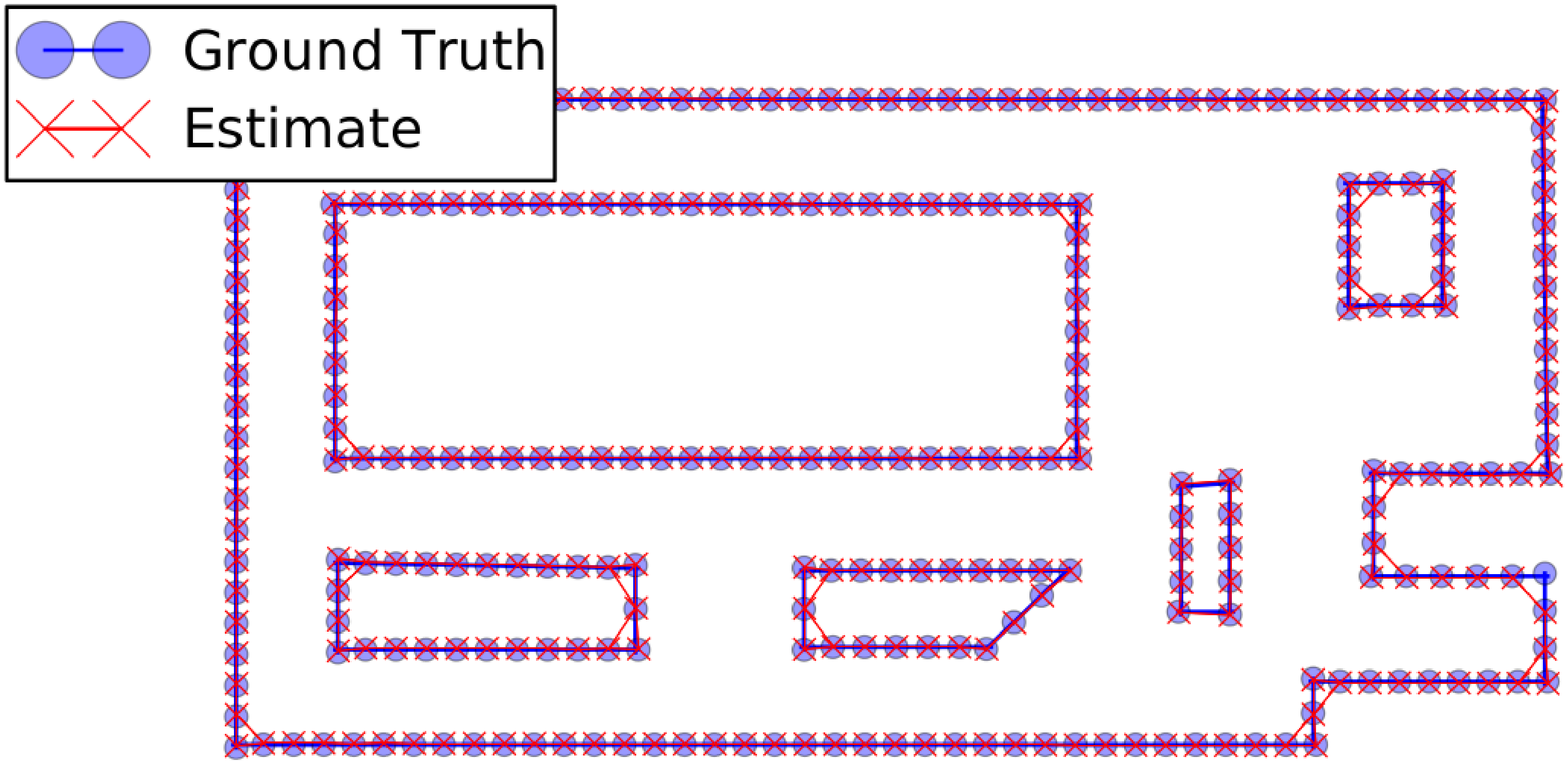}
	\caption{\small active SLAM via TFG} \label{fig:active_slam_map}
\end{subfigure}
\caption{\small SLAM result comparison. When the odometry drifted, frontier exploration with occupancy grid map have distorted maps. While active SLAM using TFG is able close loops on features and have much more accurate maps.}
\end{figure*}

\begin{figure}[th]
	\centering
	\includegraphics[height=2.5in]{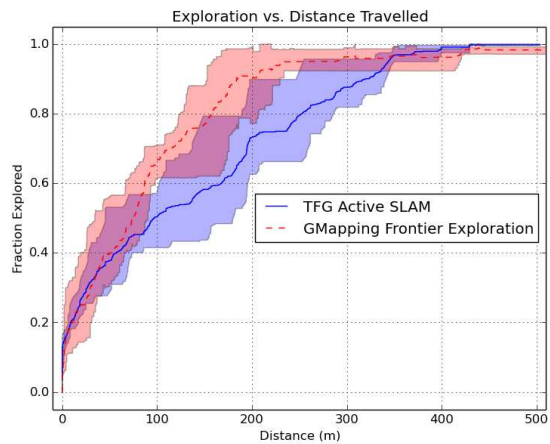}
	\caption{\small Map coverage vs distance travelled. TFG builds an accurate map while exploring and is thus slightly slower than greedy nearest-frontier exploration.} \label{fig:exploration_comparison}
\end{figure}

\begin{figure}[th]
	\centering
	\includegraphics[height=2.8in]{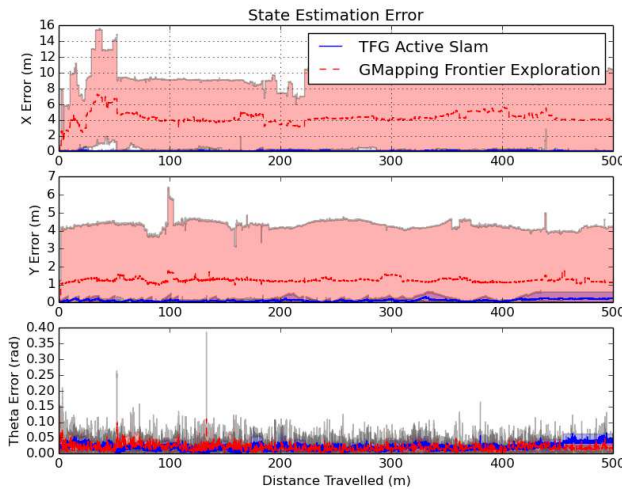}
	\caption{\small Robot pose error over multiple runs. Solid line repsents mean and shade represent range. TFG has consistently smaller errors.} \label{fig:error}
\end{figure}

\subsection{Hardware}
The new framework is tested in an indoor space with the TurtleBot platform, using a computer with specifications listed in Table \ref{tab:hardware}. The computational resources used are readily available in many modern on-board systems. The focus of this paper is not on feature detection or data association, thus april tags\cite{april} are used as features in the indoor space. Figure \ref{fig:hard_view} gives some example views of the environment. Figure \ref{fig:hard_path} shows how the robot's mapping progressed throughout the experiment. It started with a partial map, then gradually picked up the frontiers and expanded the map to cover the space. The black lines are obstacles and black dots are features. The red dot is the robot's current position and the red lines are its planned trajectories. 

\begin{table}[t]
	\centering
	\caption{Hardware Specification}\label{tab:hardware}
\begin{tabular}{|c|c|}
	\hline
	Robot & TurtleBot (Kobuki base) \\ 	\hline
	Processor & Intel Core i3 dual~@2.3GHz \\ 	\hline
	RAM	& 4GB \\ 	\hline
	Operating System & Ubuntu 14.04 \\ 	\hline
\end{tabular}
\end{table}

\begin{figure*}[t]
	\centering
\begin{subfigure}[b]{0.35\columnwidth}
	\includegraphics[height=0.9in, trim = 70 20 50 0, clip]{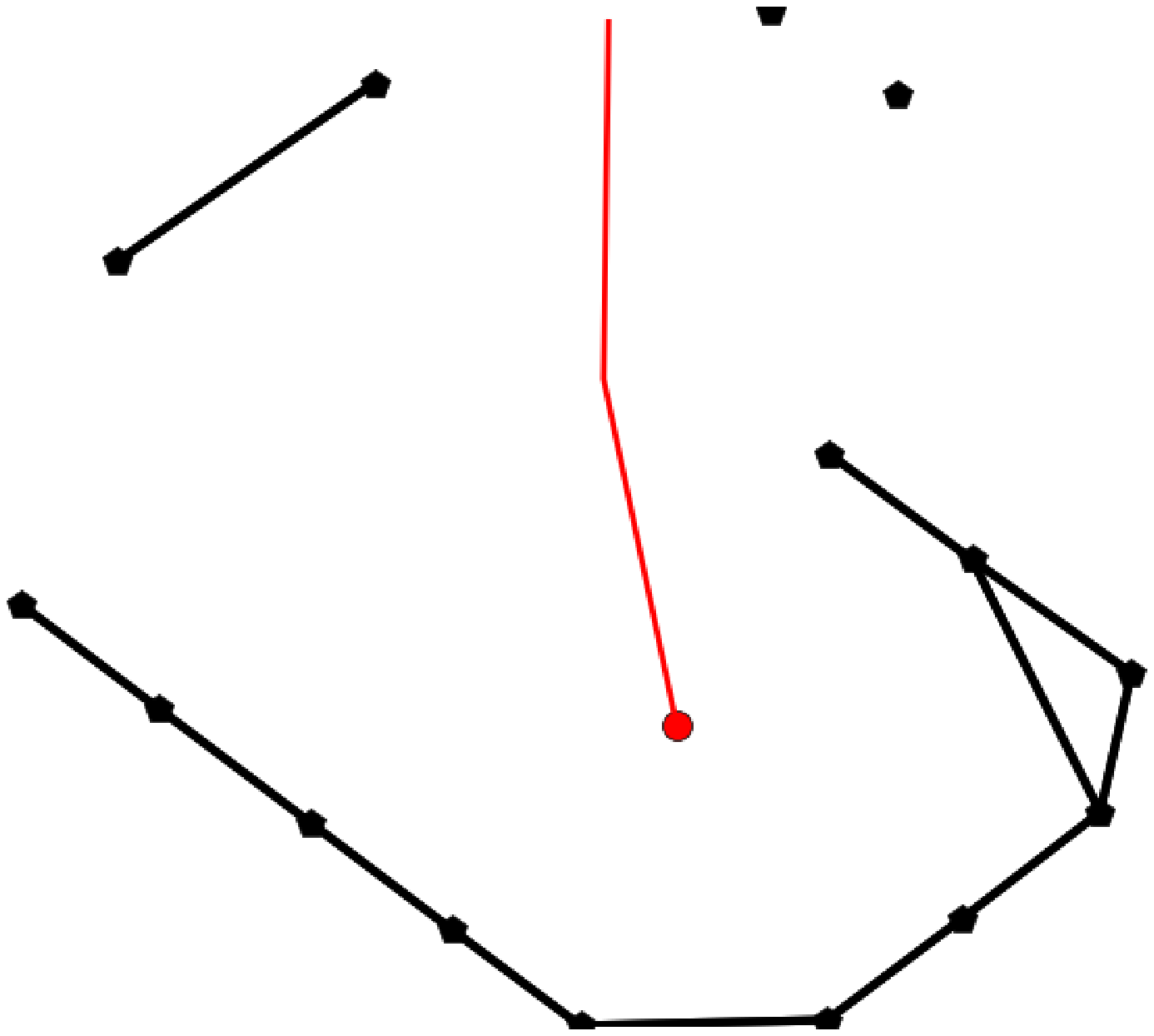} \caption{}
\end{subfigure}
\begin{subfigure}[b]{0.35\columnwidth}
	\includegraphics[height=1.0in, trim = 70 20 50 0, clip]{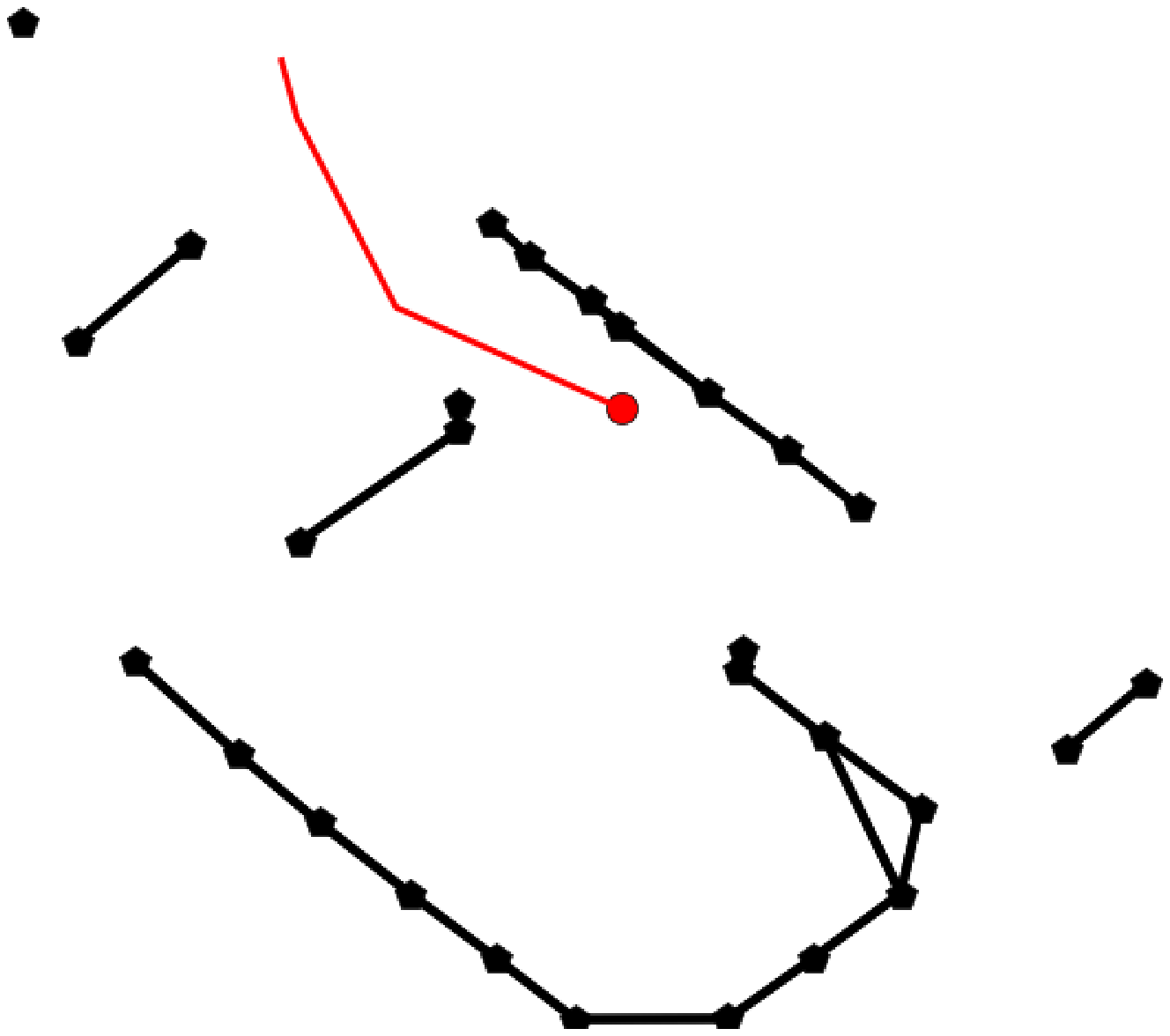} \caption{}
\end{subfigure}	
\begin{subfigure}[b]{0.35\columnwidth}
	\includegraphics[height=1.1in, trim = 70 20 50 0, clip]{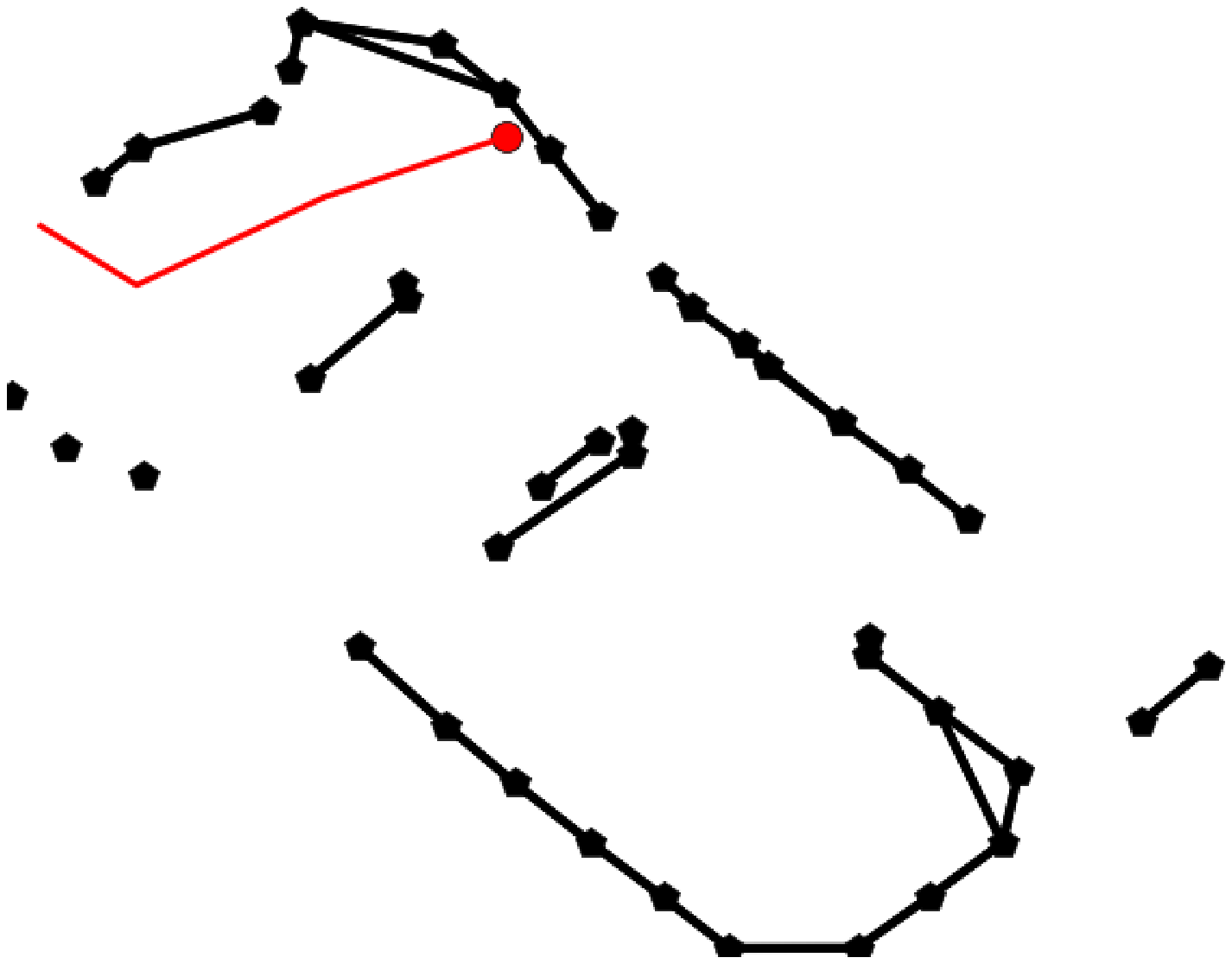} \caption{}
\end{subfigure}	
\begin{subfigure}[b]{0.35\columnwidth}
	\includegraphics[height=1.1in, trim = 70 20 50 0, clip]{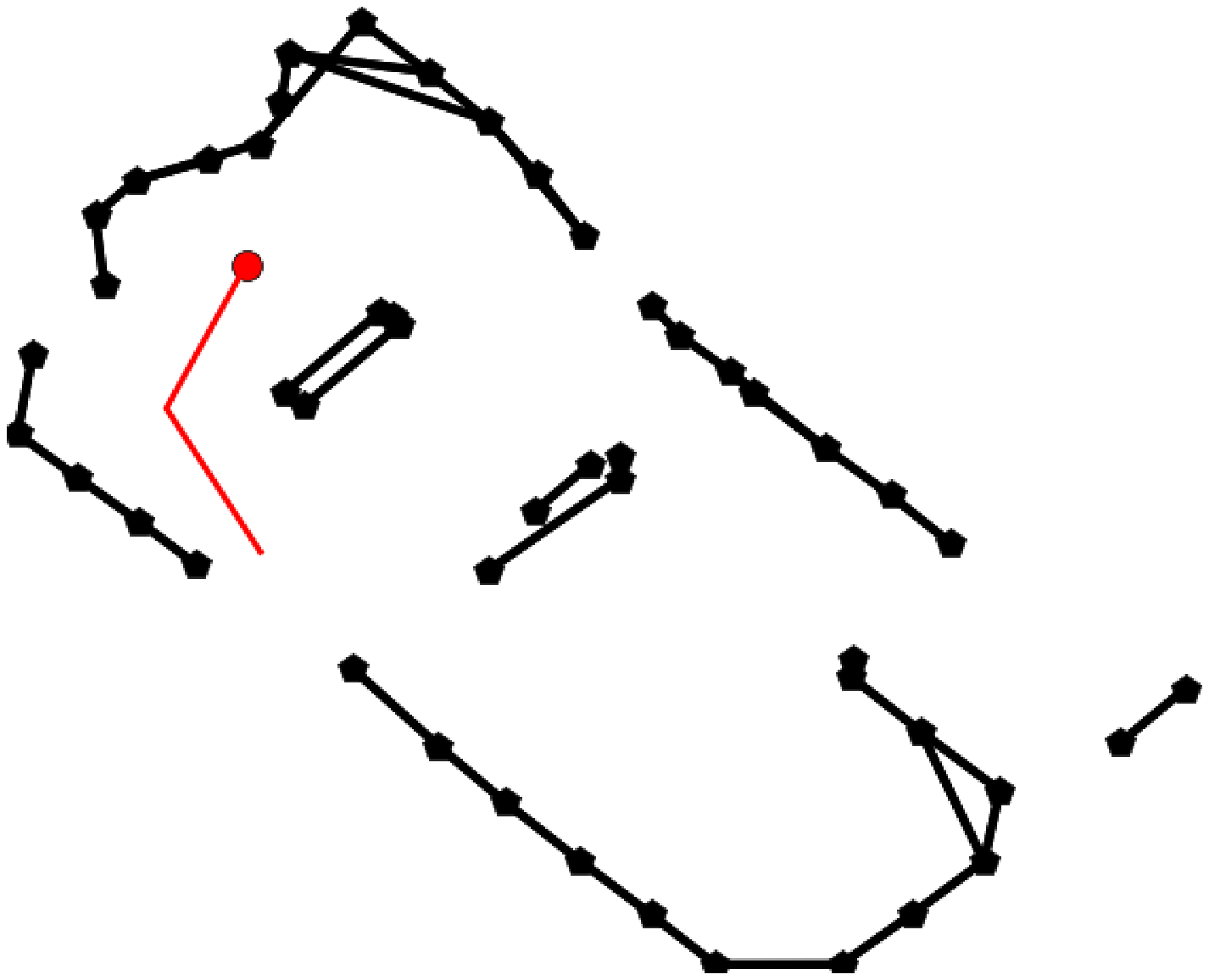} \caption{}
\end{subfigure}
\begin{subfigure}[b]{0.35\columnwidth}
	\includegraphics[height=1.1in, trim = 70 20 50 0, clip]{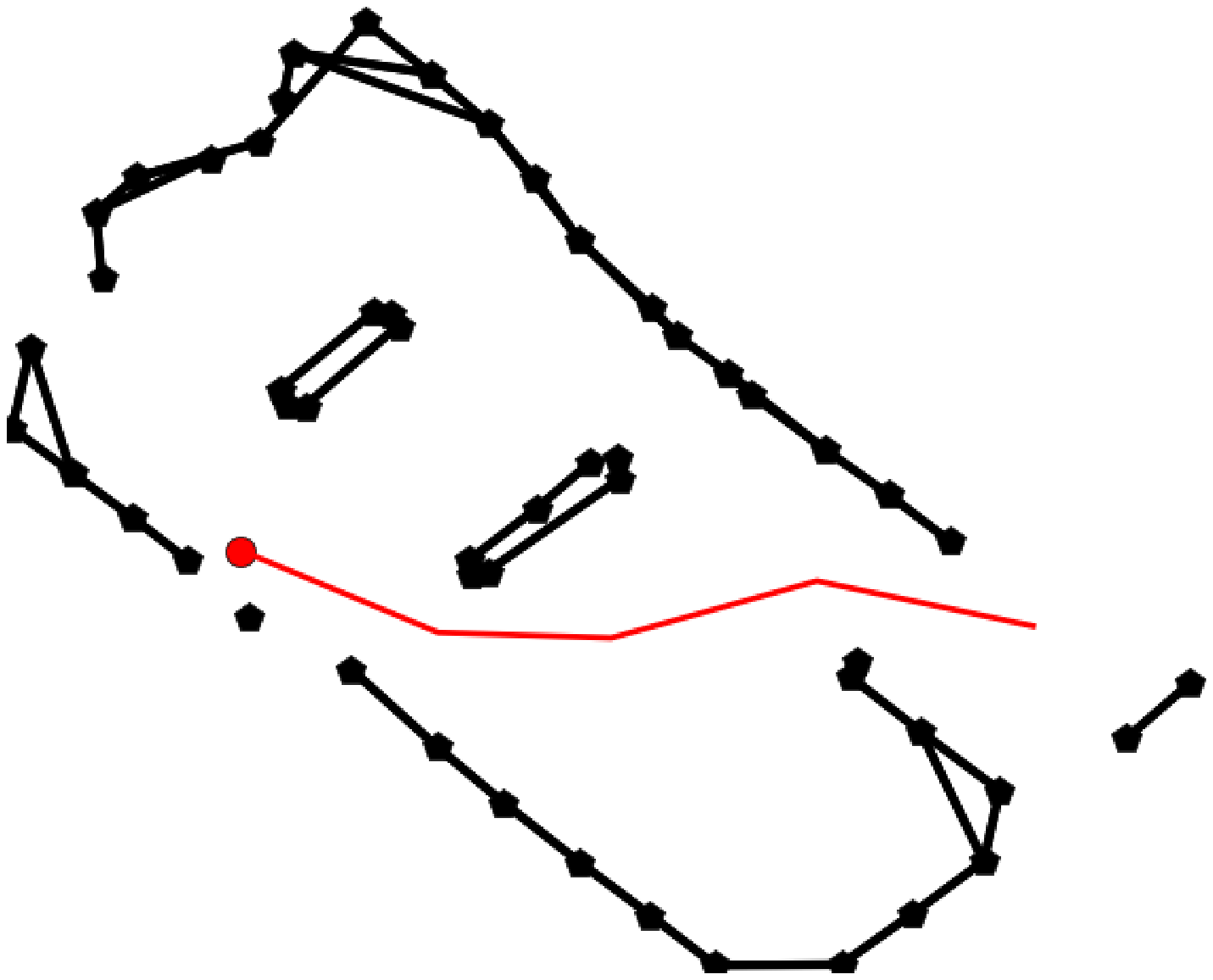} \caption{}
\end{subfigure}
\begin{subfigure}[b]{0.35\columnwidth}
	\includegraphics[height=1.1in, trim = 70 20 50 0, clip]{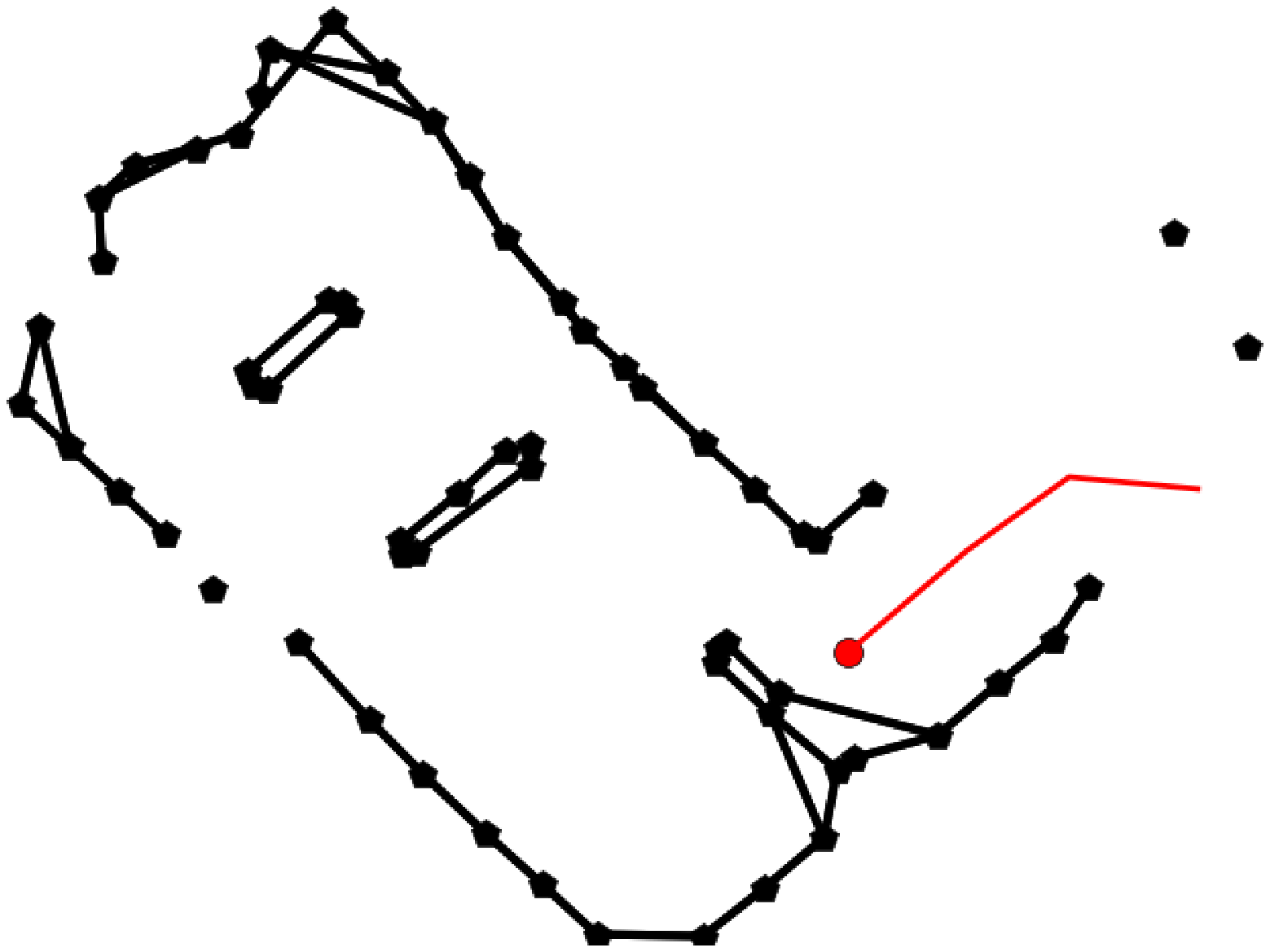} \caption{}
\end{subfigure}	
\begin{subfigure}[b]{0.35\columnwidth}
	\includegraphics[height=1.1in, trim = 70 20 50 0, clip]{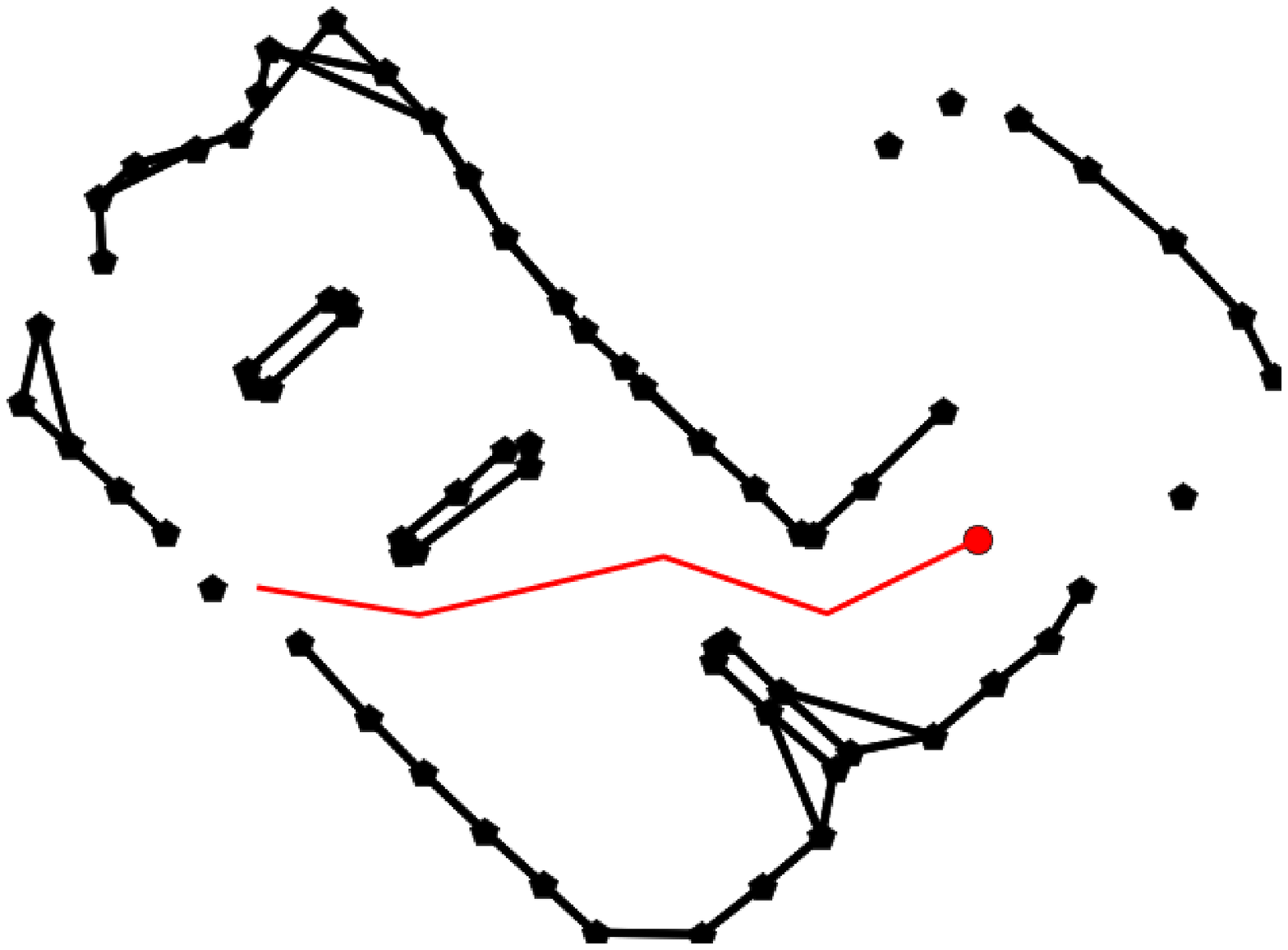} \caption{}
\end{subfigure}	
\begin{subfigure}[b]{0.35\columnwidth}
	\includegraphics[height=1.1in, trim = 70 20 50 0, clip]{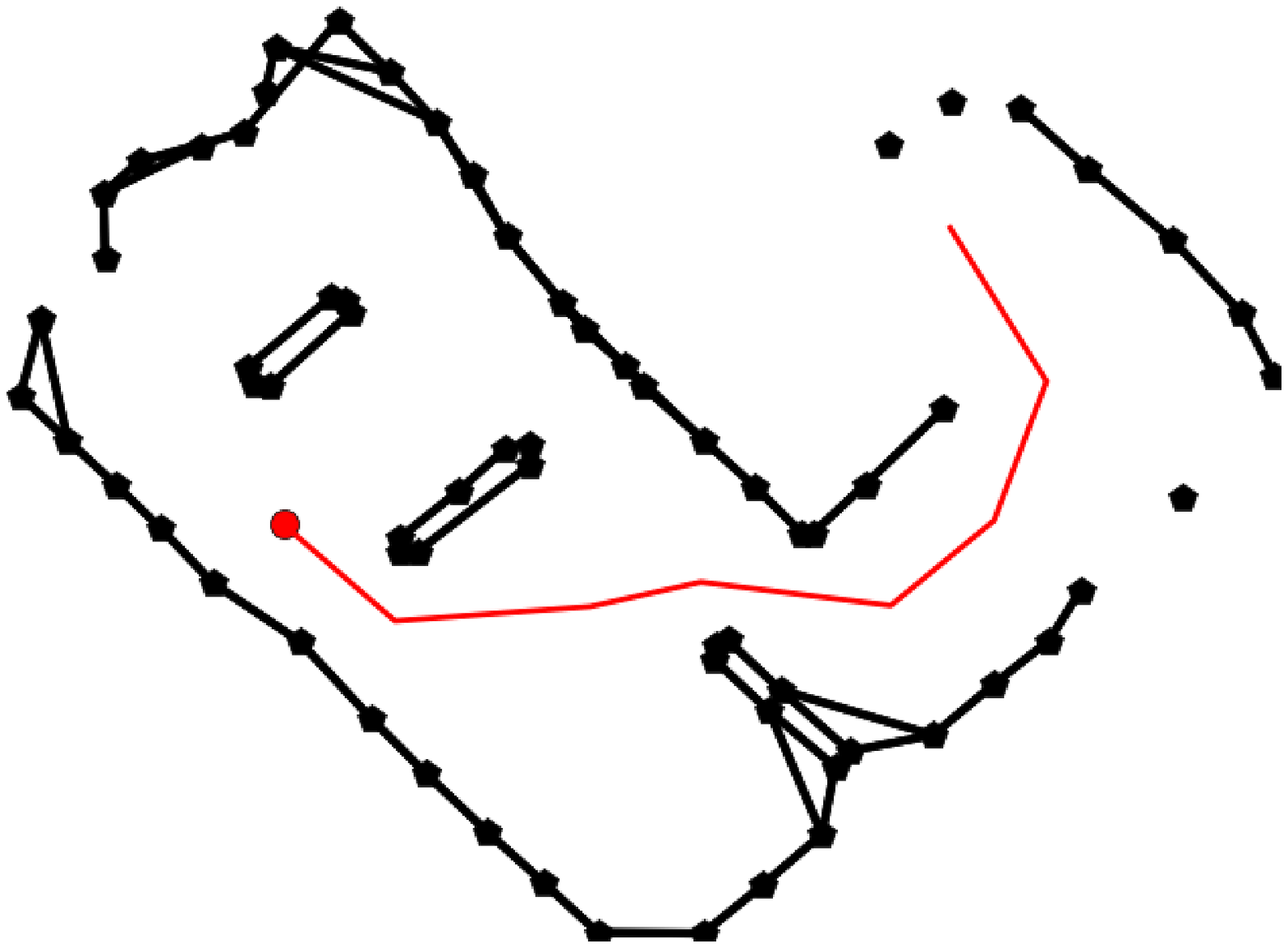} \caption{}
\end{subfigure}	
\begin{subfigure}[b]{0.35\columnwidth}
	\includegraphics[height=1.1in, trim = 70 20 50 0, clip]{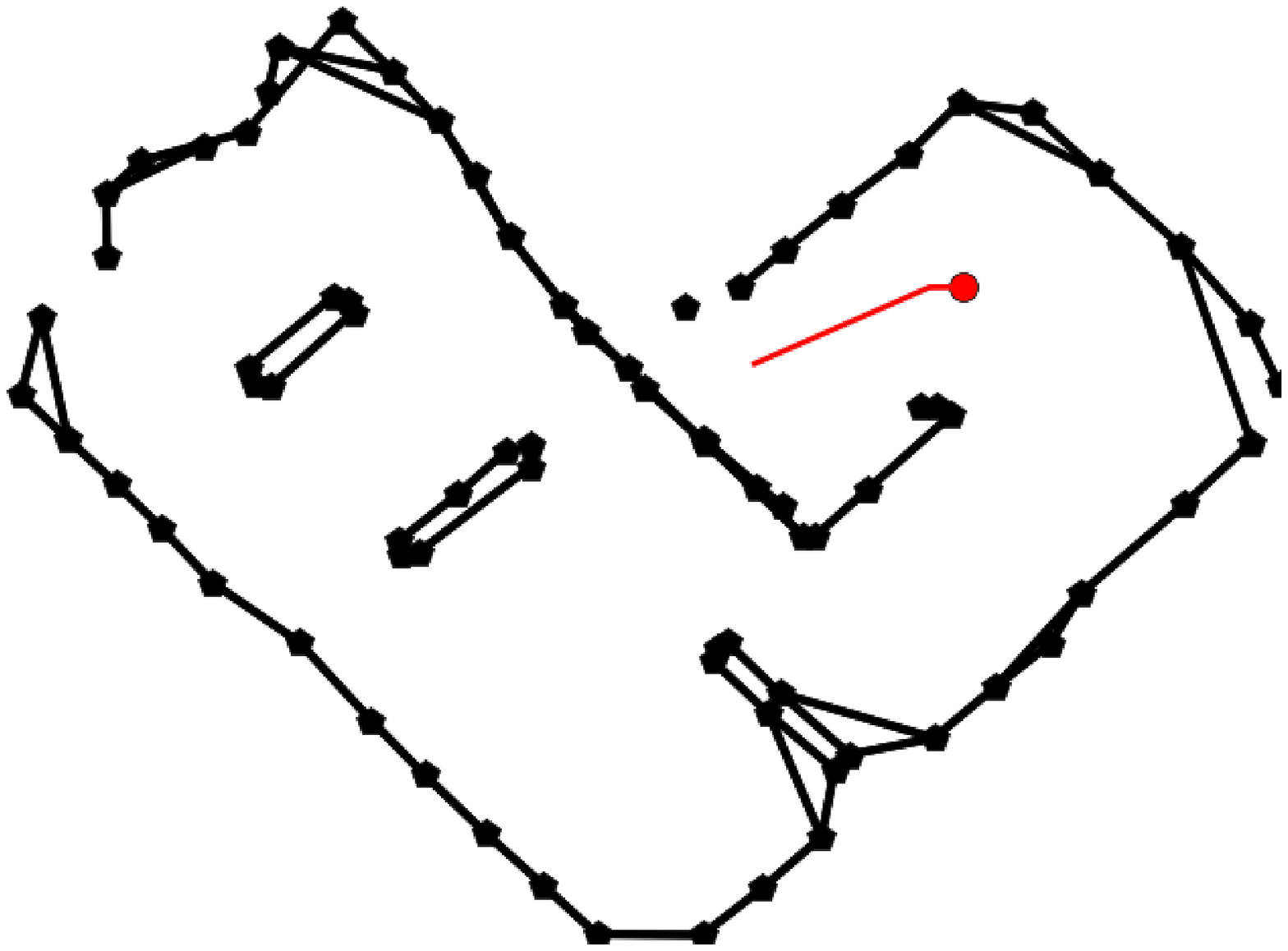} \caption{}
\end{subfigure}	
\begin{subfigure}[b]{0.35\columnwidth}
	\includegraphics[height=1.1in, trim = 70 20 50 0, clip]{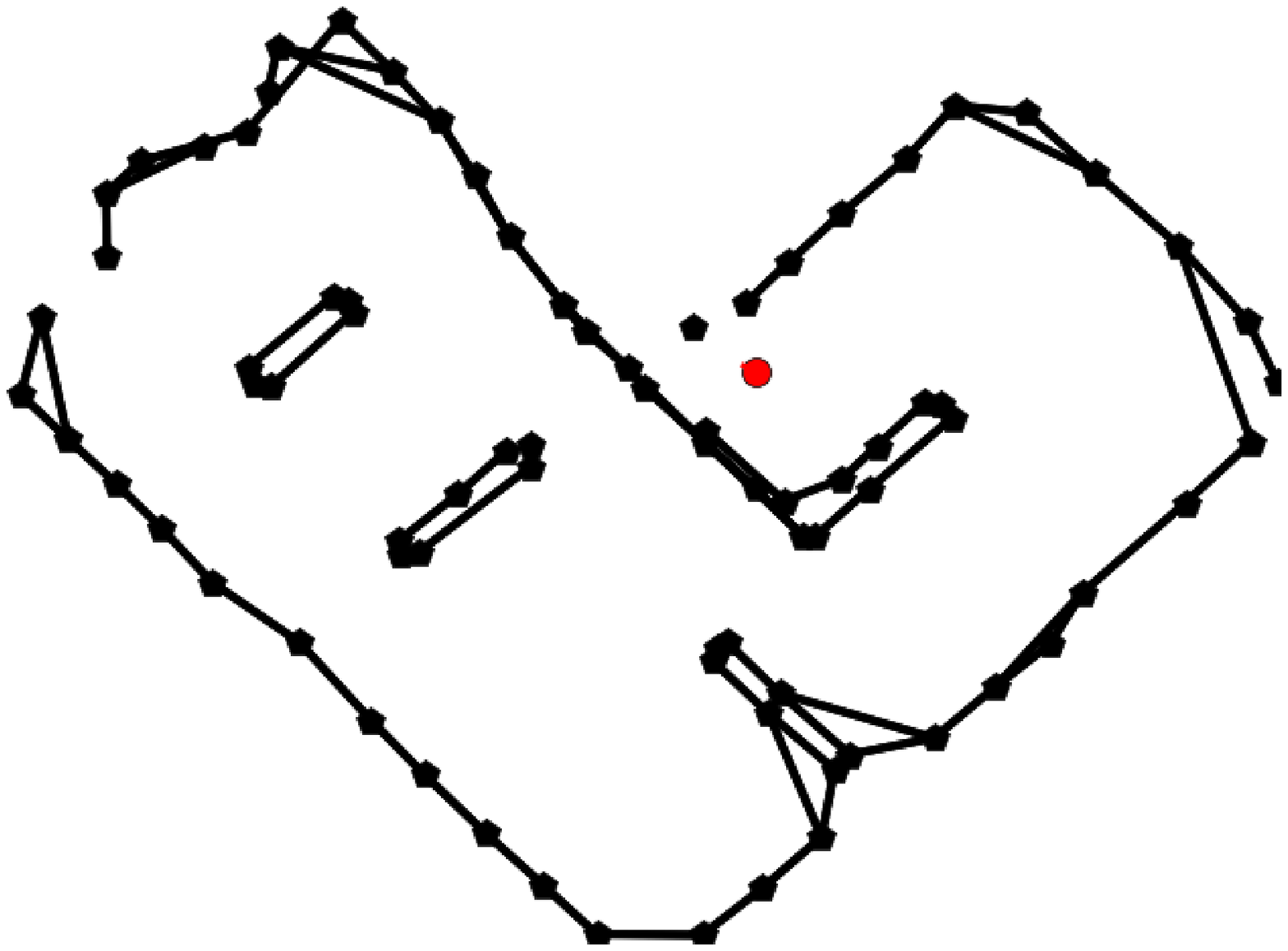} \caption{}
\end{subfigure}	
\caption{\small Robot path and TFG in hardware experiment. Black stars represent april tags, and black lines represent obstacles. The red circle represents the robot's current location, and the red line represents robot's planned trajectory. The robot started with a partial map, then gradually picked up the frontiers and expanded the map to cover the space.} \label{fig:hard_path}
\end{figure*}

\begin{figure*}[t]
	\centering
		\includegraphics[width=0.45\columnwidth]{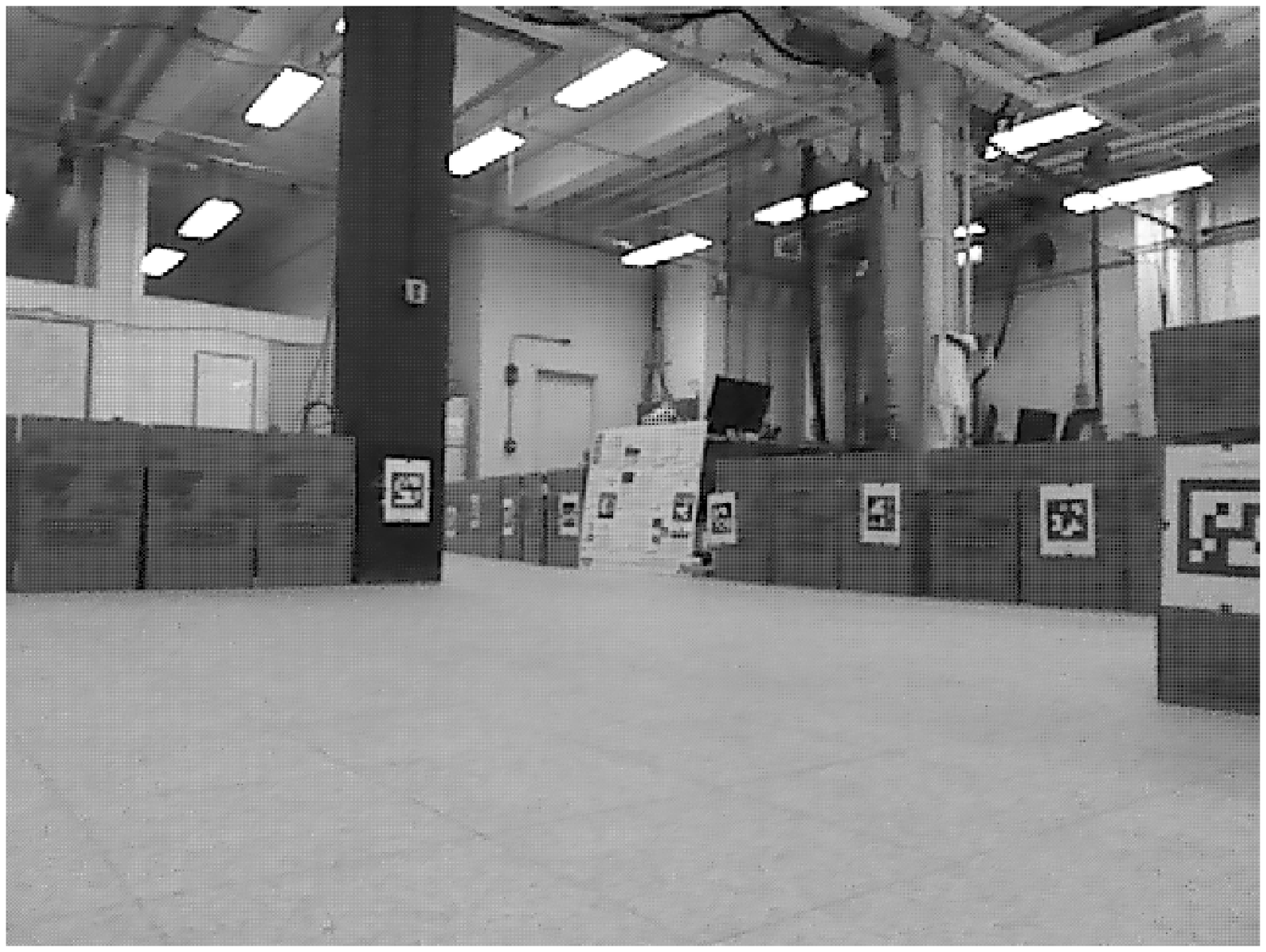}
		\includegraphics[width=0.45\columnwidth]{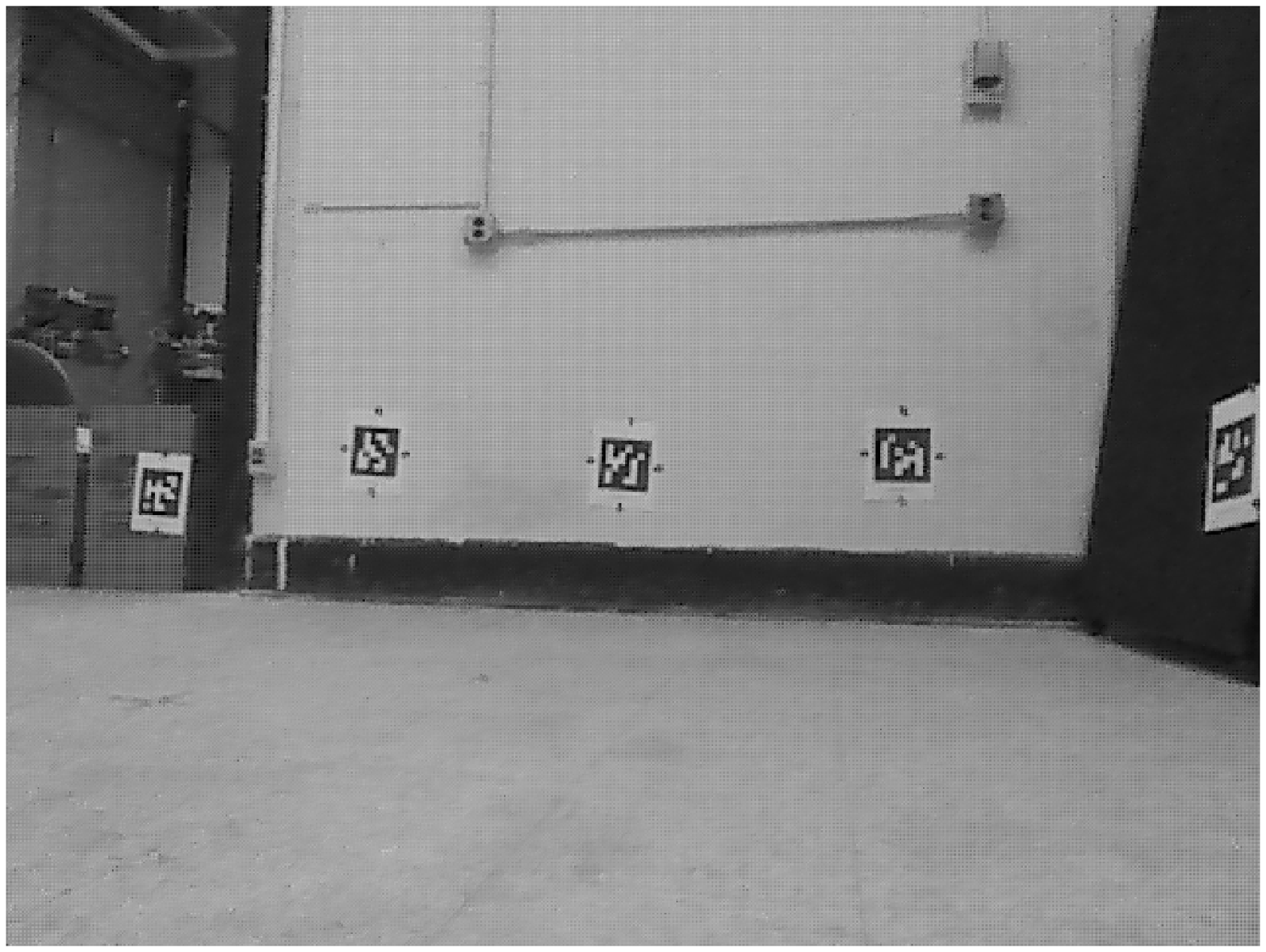}
		\includegraphics[width=0.45\columnwidth]{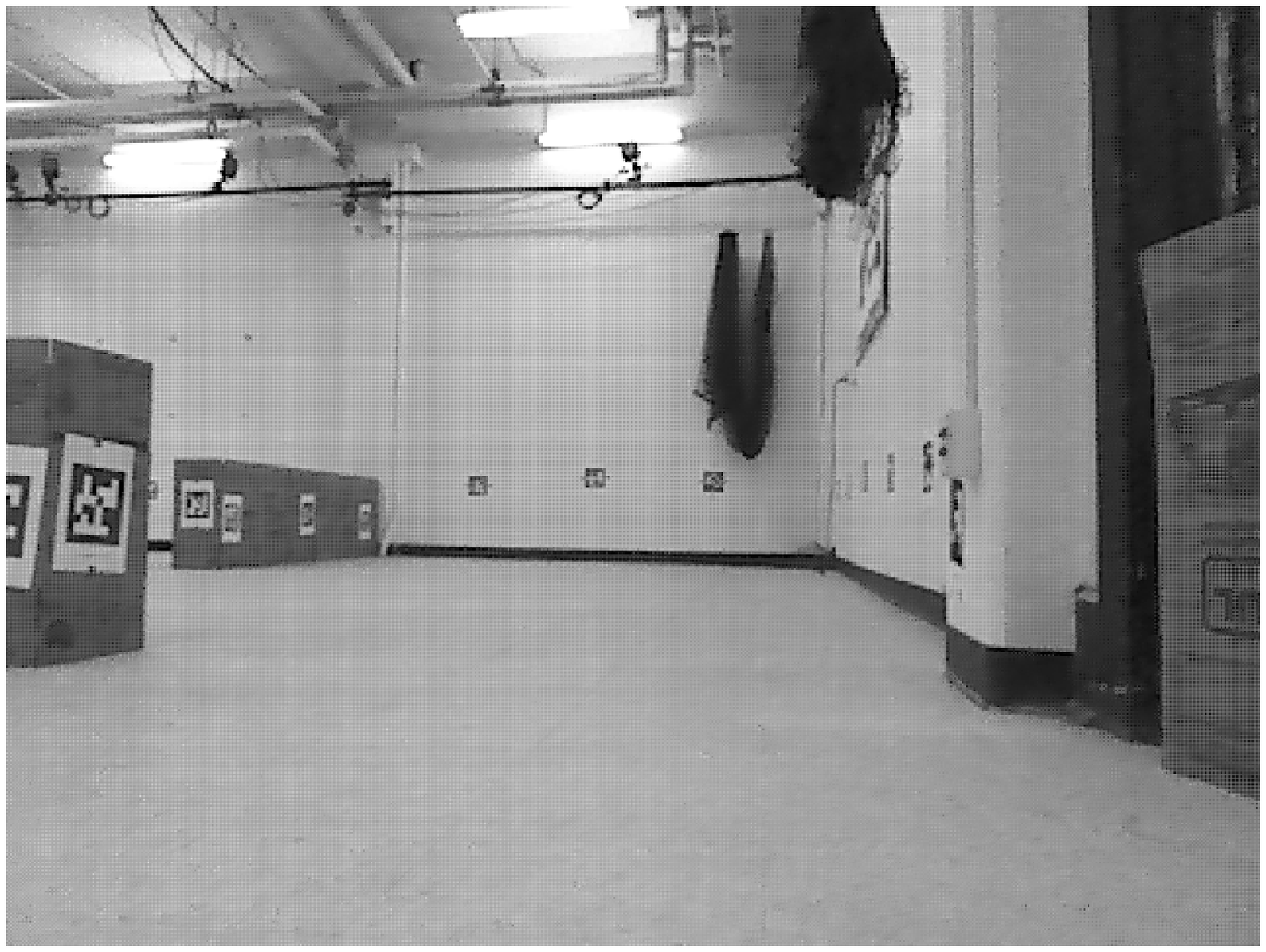}
		\includegraphics[width=0.45\columnwidth]{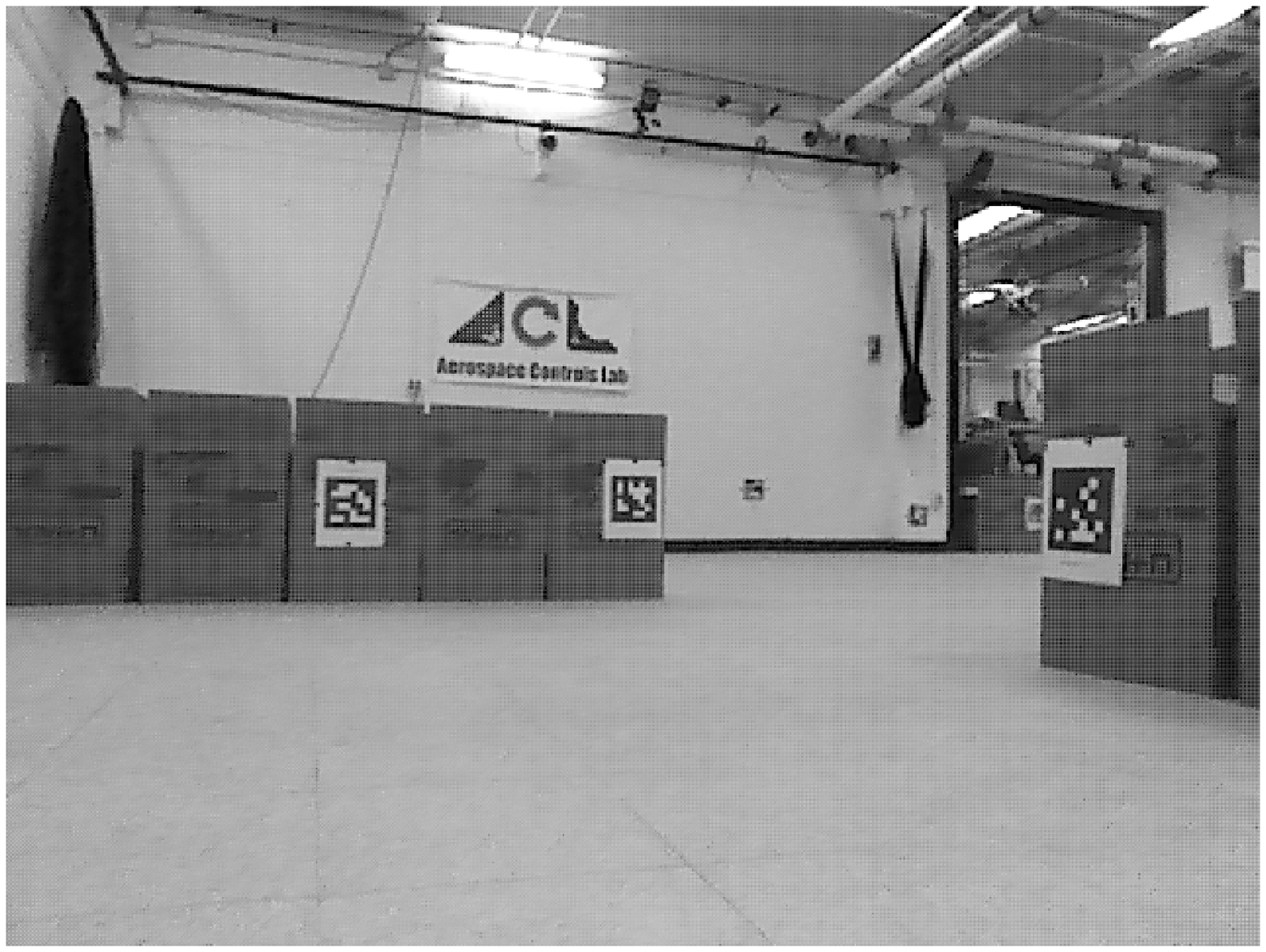}
	\caption{\small Views of the space. An GPS-denied indoor environment with april tags as features.}\label{fig:hard_view}	
\end{figure*}

Figure \ref{fig:map_compare} shows the maps generated by three different methods. The proposed TFG-based active SLAM method (Figure \ref{fig:map_AS}) effectively recovers the space even with disturbance. The map generated by human-operated data (Figure \ref{fig:map_human}) captures the basic structure, but missed some obstacles on the wall and some surfaces on the obstacles in the middle of the room. This is mainly because human operators do not have a metric model of the environment, and cannot tell if some place is well explored. Finally, there is significant distortion in grid maps (Figure \ref{fig:map_grid}). After the disturbance, the map completely drifted. Laser scans do not use any features in the environment as landmarks, thus once the odometry has drifted, it is very hard to correct even if the robot comes back to the original place.

\begin{figure*}[t]
\begin{subfigure}[b]{0.32\textwidth}
	\centering	\includegraphics[height=1.5in, trim = 70 0 50 0,clip]{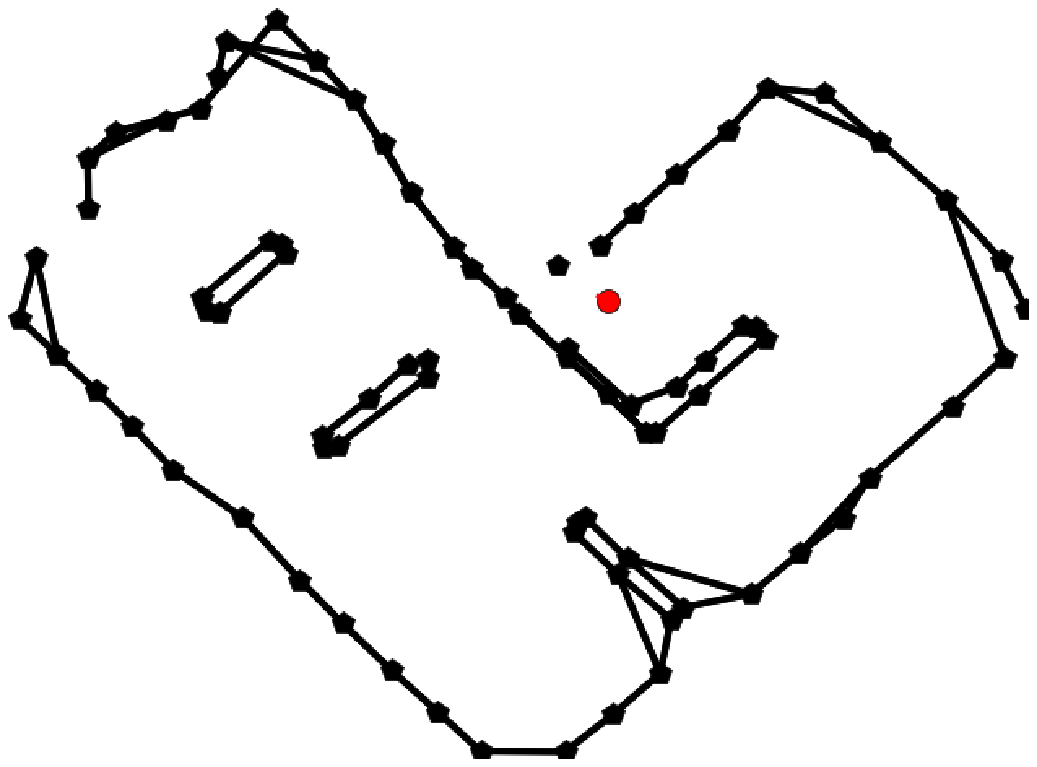} \caption{Active Slam}\label{fig:map_AS}
\end{subfigure}			
\begin{subfigure}[b]{0.32\textwidth}
	\centering	\includegraphics[height=1.5in, trim = 70 0 50 0,clip]{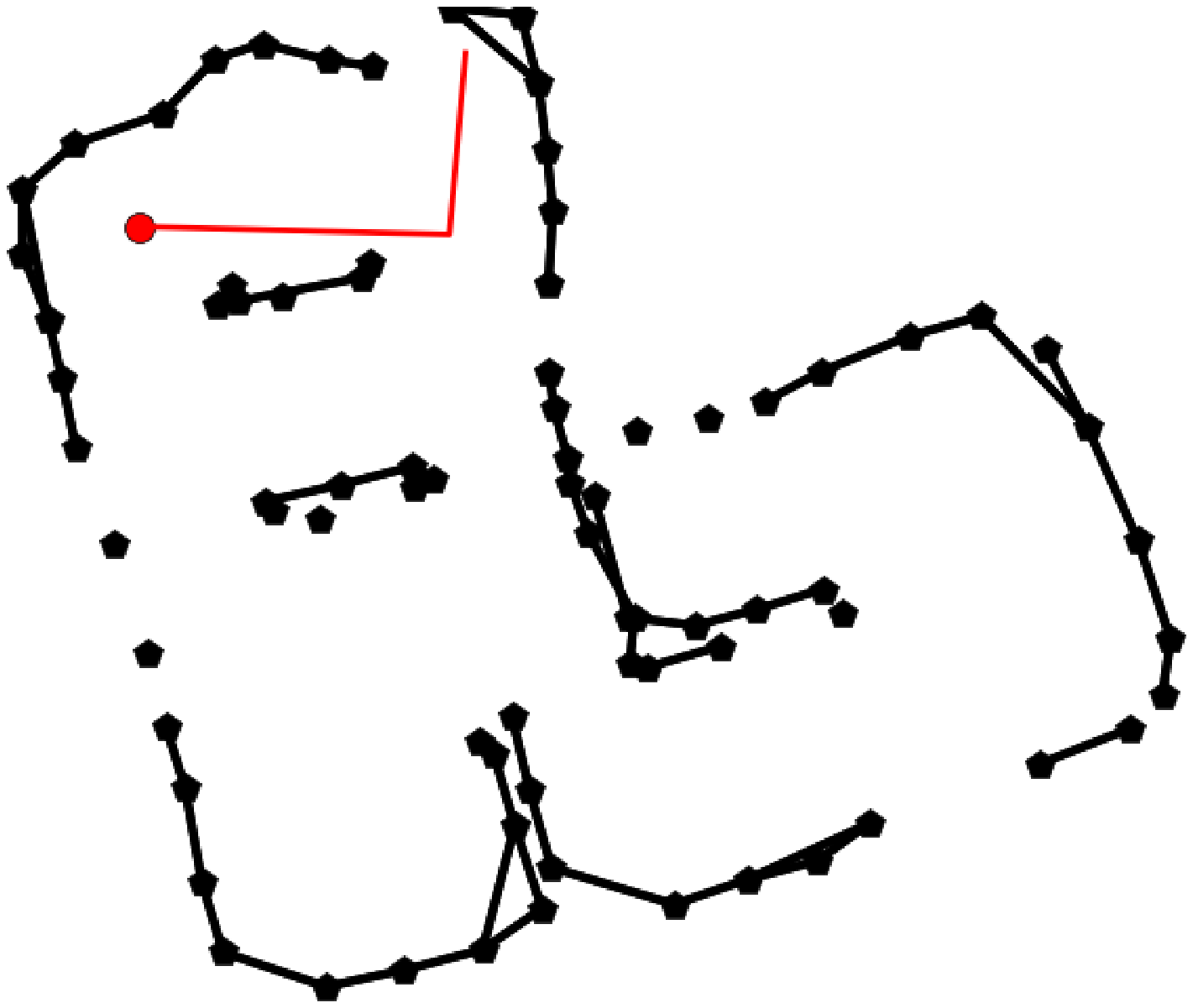} \caption{Expert Operated}\label{fig:map_human}
\end{subfigure}	
\begin{subfigure}[b]{0.35\textwidth}
	\centering	\includegraphics[height=1.5in, trim = 0 0 0 0,clip]{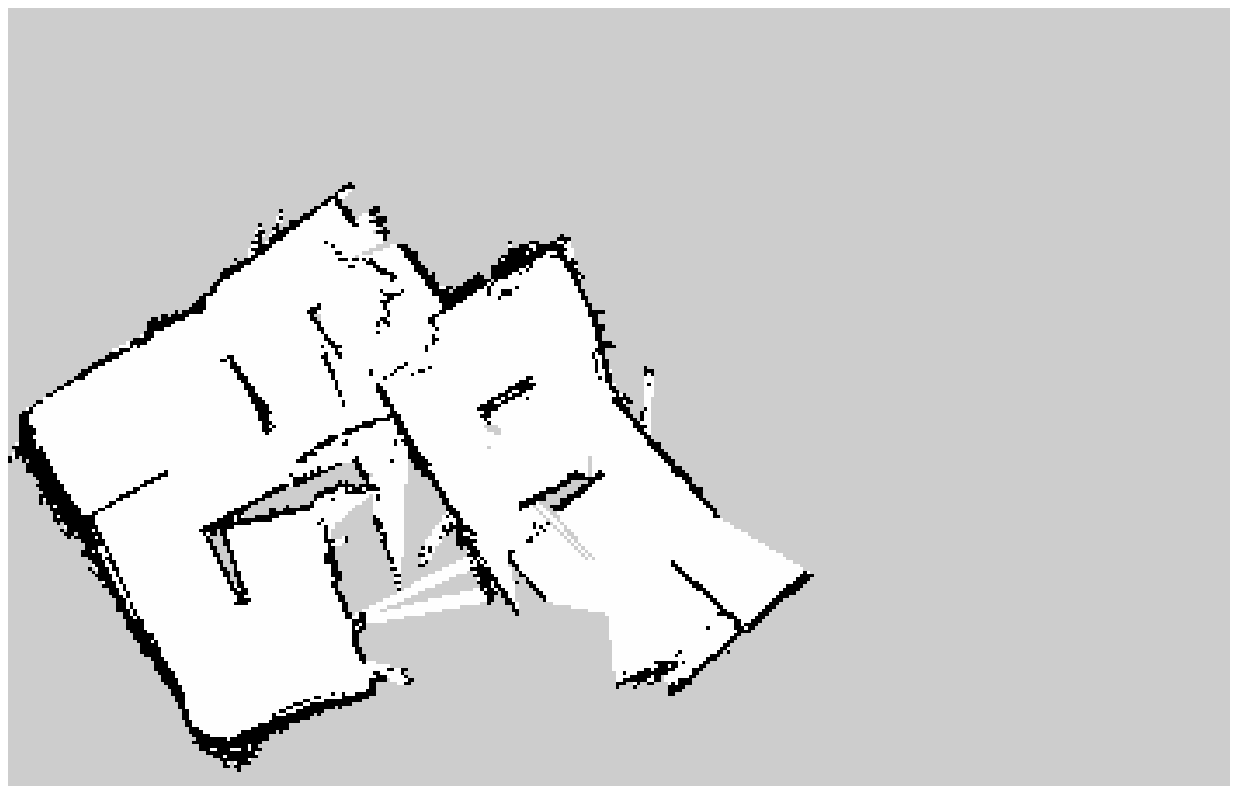} \caption{Grid map}\label{fig:map_grid}
\end{subfigure}	
\caption{\small Comparison of different policies. TFG-based active SLAM method effectively recovers the space even with disturbance. The map by human-operated data captures the basic structure, but missed some obstacles on the wall and some surfaces on the obstacles in the middle of the room. The grid map has significant distortion and completely drifted after the disturbance.}\label{fig:map_compare}
\end{figure*}

\section{Conclusion} \label{sec:conclusion}
This paper contributed to the problem of active simultaneous localization and mapping (SLAM) in three ways:
\begin{enumerate}
	\item proposed a topological feature graph (TFG) that extends point estimates in SLAM to geometry representation of the space.
	\item An information objective that directly quantifies uncertainty of a TFG. It captures correlations between robot poses and features in the space in a unified framework, thus new feature observations can help close loops and reduce uncertainties on observed features. The exploration and exploitation naturally comes out of the framework for a given feature density.
	\item An efficient sampling-based path planning procedure within the TFG, which enables active SLAM.
\end{enumerate}

Future work includes extending the algorithm to visual feature/object detection and association, and extending the TFG to work with 3D active SLAM.

\section*{ACKNOWLEDGMENTS}
\noindent This research is supported in part by ARO MURI grant W911NF-11-1-0391, ONR grant N00014-11-1-0688 and NSF Award IIS-1318392.

\clearpage 
\balance
\bibliographystyle{unsrt}
\bibliography{AliAgha,Beipeng,liam_refs,BIB_all/ACL_all,BIB_all/ACL_Publications,BIB_all/ACL_bef2000}


\end{document}